\pdfoutput=1

\documentclass[twocolumn]{article}

\usepackage{times}
\usepackage{graphicx} % more modern
\usepackage{epsfig} % less modern
\usepackage[numbers]{natbib}

\usepackage{algorithm}
\usepackage{algorithmic}
\usepackage{thmtools, thm-restate}

\usepackage{hyperref}
\usepackage{booktabs}

\usepackage{tkz-graph}
\usetikzlibrary{shapes,fit,arrows,positioning}
\tikzstyle{vertex} = [circle, draw, thick, text centered]
\tikzstyle{edge} = [draw, thick,-]
\usepackage[caption=false,font=footnotesize]{subfig} 
\usepackage{verbatim}
\usepackage{textcomp}
\usepackage{amsmath,amsthm,amssymb}
\usepackage{bbm}
\usepackage{dsfont}
\usepackage{mathtools,enumerate}
\usepackage{breqn}
\usepackage{xcolor}
\usepackage{hyperref}

\usepackage{array}

\usepackage
[
        a4paper,% other options: a3paper, a5paper, etc
        left=2cm,
        right=2cm,
        top=3cm,
        bottom=4cm,
]
{geometry}

\DeclareMathOperator\erf{erf}
\DeclareMathOperator\ppd{ppd}
\DeclareMathOperator\up{up}
\DeclareMathOperator\down{down}
\DeclareMathOperator\prv{prv}
\DeclareMathOperator{\argmax}{arg\,max}

\newcommand{\unicode}[1]{{}}

\newtheorem{corollary}{Corollary}

\theoremstyle{definition}
\newtheorem{definition}{Definition}
 
\theoremstyle{remark}
\newtheorem*{remark}{Remark}

\theoremstyle{plain}

\newcommand\RBM{\text{RBM}}
\newcommand\smallO{
  \mathchoice
    {{\scriptstyle\mathcal{O}}}% \displaystyle
    {{\scriptstyle\mathcal{O}}}% \textstyle
    {{\scriptscriptstyle\mathcal{O}}}% \scriptstyle
    {\scalebox{.7}{$\scriptscriptstyle\mathcal{O}$}}%\scriptscriptstyle
  }

\begin{document} 

\title{Using Inherent Structures to design \\
            Lean 2-layer RBMs}

% It is OKAY to include author information, even for blind
% submissions: the style file will automatically remove it for you
% unless you've provided the [accepted] option to the icml2017
% package.

% list of affiliations. the first argument should be a (short)
% identifier you will use later to specify author affiliations
% Academic affiliations should list Department, University, City, Region, Country
% Industry affiliations should list Company, City, Region, Country

% you can specify symbols, otherwise they are numbered in order
% ideally, you should not use this facility. affiliations will be numbered
% in order of appearance and this is the preferred way.
%\icmlsetsymbol{equal}{*}

\author{
  Abhishek Bansal\\
    \texttt{IBM Research}\\
  \texttt{abbansal@in.ibm.com}
  \and
  Abhinav Anand \\  
          \texttt{IISc Bengaluru}\\
  \texttt{abhinava@iisc.ac.in}
    \and
    Chiranjib Bhattacharyya\\
            \texttt{IISc Bengaluru}\\
    \texttt{chiru@iisc.ac.in}
}
%\begin{icmlauthorlist}
%\icmlauthor{Abhishek Bansal}{to}
%\icmlauthor{Abhinav Anand}{ed}
%\icmlauthor{Chiranjib Bhattacharyya}{ed}
%\end{icmlauthorlist}

%\icmlaffiliation{to}{IBM Research}
%\icmlaffiliation{ed}{Dept of CSA, IISc, Bengaluru, India}

%\icmlcorrespondingauthor{Abhishek Bansal}{abbansal@in.ibm.com}
%\icmlcorrespondingauthor{Abhinav Anand}{abhinava@iisc.ac.in}
%\icmlcorrespondingauthor{Chiranjib Bhattacharyya}{chiru@iisc.ac.in}

% You may provide any keywords that you 
% find helpful for describing your paper; these are used to populate 
% the "keywords" metadata in the PDF but will not be shown in the document
%\icmlkeywords{boring formatting information, machine learning, ICML}

\vskip 0.3in
%]

% this must go after the closing bracket ] following \twocolumn[ ...

% This command actually creates the footnote in the first column
% listing the affiliations and the copyright notice.
% The command takes one argument, which is text to display at the start of the footnote.
% The \icmlEqualContribution command is standard text for equal contribution.
% Remove it (just {}) if you do not need this facility.

%\printAffiliationsAndNotice{}  % leave blank if no need to mention equal contribution
%\printAffiliationsAndNotice{\icmlEqualContribution} % otherwise use the standard text.
\newcommand{\note}[1]{\textcolor{red}{#1}}
\newcommand{\Lim}[1]{\raisebox{0.5ex}{\scalebox{0.8}{$\displaystyle\lim_{#1}\;$}}}

\def\ISC{{\textbf{ISC}}}
\def\EE{\mathds{E}}
\def\RR{\mathds{R}}
\def\O1{\mathds{1}}
\def\C{{\mathcal C}}
\def\L{{\mathcal L}}
\def\H{{\mathcal H}}
\def\M{{\mathcal M}}
\def\ppd{{\tt PPD}}
\def\bv{\textbf{v}}
\def\bh{\textbf{h}}
\def\ba{\textbf{a}}
\def\bb{\textbf{b}}
\def\A{{\mathcal A}}
\twocolumn[
  \begin{@twocolumnfalse}
    \maketitle
  \end{@twocolumnfalse}
]
%\maketitle

%\setlength{\columnsep}{20pt}
%\begin{multicols}{2}

% As a general rule, do not put math, special symbols or citations
% in the abstract
\begin{abstract}
Understanding the representational power of  Restricted Boltzmann Machines (RBMs) with multiple layers is  an ill-understood problem and is an area of active research. Motivated from the approach of \emph{Inherent Structure formalism} \cite{stillinger1982hidden}, extensively used in analysing Spin Glasses, we propose a novel measure called \emph{Inherent Structure Capacity} ($ \ISC $), which characterizes the representation capacity of a fixed architecture RBM by the expected number of modes of distributions emanating from the RBM with parameters drawn from a prior distribution. 
 Though $ \ISC $ is intractable, we show that for a single layer RBM architecture $\ISC$ approaches a finite constant as number of hidden units are increased and to further improve the $\ISC$, one needs to add a second layer. 
Furthermore, we introduce \emph{Lean} RBMs, which are multi-layer RBMs where each layer can have at-most $O(n)$ units with the number of visible units being ~$n$.  
We show that for every single layer RBM with $\Omega(n^{2+r}), r \ge 0$,   hidden units there exists  a two-layered \emph{lean} RBM with $ \Theta(n^2) $ parameters with the same $\ISC$, establishing  that 2 layer RBMs can achieve the same representational power as single-layer RBMs but using far fewer number of parameters.  
To the best of our knowledge, this is the first result which quantitatively establishes the need for layering.
\end{abstract}

% no keywords

% For peer review papers, you can put extra information on the cover
% page as needed:
% \ifCLASSOPTIONpeerreview
% \begin{center} \bfseries EDICS Category: 3-BBND \end{center}
% \fi
%
% For peerreview papers, this IEEEtran command inserts a page break and
% creates the second title. It will be ignored for other modes.
%\IEEEpeerreviewmaketitle

\section{Introduction}
% no \IEEEPARstart
Deep Boltzmann Machines (DBMs\footnote{We shall use the terms RBM and DBM interchangebly}) are largely tuned using empirical methods based on trial and error. Despite much effort, there is still very little theoretical understanding about why a particular neural network architecture works better than another for any given application. Furthermore there is no well defined metric to compare different network architectures. 

It is known that given any input distribution on the set of binary vectors of length $ n $, there exists an RBM with $ \alpha 2^n-1 $ ($ \alpha < 1 $) hidden units that can approximate that distribution to an arbitrary precision \cite{montufar2017hierarchical}. However with these many hidden units the number of parameters increase exponentially. We call a network \emph{lean} if for each layer, the number of hidden units $ m = \mathcal{O}(n) $ where $ n $ is the number of visible units. The deep narrow Boltzmann Machines whose universal approximation properties were studied in \cite{montufar2014deep} are a special case of lean networks. In this paper we study lean 2-layer deep RBMs.

We ask the questions, is there a measure that can relate DBM architectures to their representational power? Once we have such a measure then can we gain insights into the capabilities of different DBM architectures? 

%In particular is there 
% an efficient \footnote{We consider efficiency in terms of number of parameters.} DBM architecture that can approximate a given input distribution with $ k $ modes? 
For example, given a \emph{wide} single layer RBM, an RBM with many hidden nodes, can we find a \emph{lean} multilayer RBM with equivalent representational power but with far lesser parameters? 
Despite much effort these questions are not satisfactorily answered and may provide important insights to the area of Deep Learning.

\noindent
Our main contributions are as follows:
\begin{enumerate}
\item We study the \emph{Inherent structures} formalism, first introduced in Statistical Mechanics\cite{stillinger1982hidden}, to understand the configuration space of RBMs. We introduce a capacity measure Inherent Structure Capacity ($ \ISC $) (Definition \ref{def:capacity}) and discuss its relation with the expected number of perfectly reconstructible vectors \cite{montufar2015does}, one-flip stable states and the modes of the input distribution. We use this as a measure of representation power of an RBM.
 
\item Existing methods for computing expected number of inherent structures are rooted in Statistical Mechanics. They use the \textit{replica} approach \cite{bray1980metastable} which does not extend well to DBMs since it is not straightforward to incorporate the bipartite nature and layering in the calculations. We use a first principles approach to devise a method that yields upper and lower bounds for single layered and two-layered DBMs (Theorems \ref{thm-main},\ref{thm-dbn}). We show that the bounds become tight as we increase the number of hidden units.

\item Previous results have shown that a sufficiently large single layer RBM can represent any distribution on the $ 2^n $ input visible vectors. However we show that if we continue adding units to hidden layer then the $ \ISC $ tapers to $ 0.585$ as opposed to the maximum limit of $ 1.0 $ (Corollary \ref{cor-high-alpha1-med-alpha2}). This implies that although an RBM is a universal approximator, if the input distribution contains large number of modes multi-layering should be considered. We have empirically verified that when the number of units in a single hidden layer RBM, $m \ge 20n$, the $ \ISC $ saturates (Figure \ref{fig-comparison}). 

%\item
%For a DBM with two hidden layers if hidden layer 1 has small number of units ($ m_1 $) then the capacity is bounded by $ 1.5^{m_1} $ regardless of the number of units in layer 2. Thus adding large number of units to layer 2 does not impact capacity and instead units should be added to hidden layer 1 (Corollary \ref{cor-low-alpha1}).

\item  By analyzing the $ \ISC $ for two layer RBM we obtain an interesting result that for any such RBM with $m = \Omega(n^2)$ hidden units (number of parameters $ \Omega(n^3) $) one can construct a  two layered DBM with $ 1.6n $ units in hidden layer 1 and $ 0.6n $ units in layer 2 (Corollary \ref{cor:single-double-comparison}) and  with number of parameters $ \Theta(n^{2})$, resulting in an order of magnitude saving in parameters. 
To the best of our knowledge this is the first such result which establishes 
the superiority of 2 layer DBMs over wide single layer RBMs in terms of representational efficiency.
 We conduct extensive experiments on synthetic datasets to verify our claim.
\end{enumerate}

\section{Model Definition and Notations}
\label{model}
\noindent
An RBM with $n$ visible and $m$ hidden units, denoted by $\RBM_{n,m}(\theta)$, is a probability distribution on ~$\{0,1\}^n$ 
of the form 
\begin{equation}\label{eq:rbm-prob}
P(\textbf{v},\textbf{h}|\theta) = \frac{e^{- E(\textbf{v},\textbf{h}|\theta)}}{Z(\theta)}
\end{equation}
\begin{equation}\label{eq:rbm-energy}
E(\textbf{v},\textbf{h}|\theta) =  - \textbf{a}^T\textbf{v} - \textbf{b}^T\textbf{h} - \textbf{v}^T\textbf{W}\textbf{h}
\end{equation}
\noindent
where  $ \textbf{v} \in \{0,1\}^n $ denotes the visible vector, hidden vector is denoted by $ \textbf{h} \in \{0,1\}^m $,  the parameter $\theta = \{\textbf{a},\textbf{b}, \textbf{W}\}$ denotes the set of biases $\textbf{a} \in \mathbb{R}^n, \textbf{b} \in \mathbb{R}^m $ and coupling matrix $ \textbf{W} \in \mathbb{R}^{n\times m} $ and  $ Z(\theta) =  \sum_{\textbf{v},\textbf{h}}e^{-E(\textbf{v},\textbf{h})} $ is the normalization constant. The log-likelihood of a given visible vector \textbf{v} for an $\RBM_{n,m}(\theta)$ is given by
\begin{equation}\label{eq:lklhd}
\L(\textbf{v}|\theta)= \ln P(\textbf{v}|\theta) =  \ln \sum_{\textbf{h}} e^{-E(\textbf{v},\textbf{h})} -\ln Z(\theta)
\end{equation}
In the sequel $\textbf{\RBM}_{n,m}$ will denote the family of distributions parameterized by $\RBM_{n,m}(\theta)$.

\begin{definition}[{\bf Modes}]
\label{modes}
Given a distribution $ p $ on vectors $ \{0,1\}^n $, a vector $ \textbf{v} $ is said to be a mode of that distribution if for all $ \textbf{v}' $ such that $ d_H(\textbf{v},\textbf{v}') = 1$, $ p(\textbf{v}) > p(\textbf{v}') $. Here $d_H(\textbf{v},\textbf{v}')= \sum_{i=1}^N [1-\delta(\textbf{v}_i,\textbf{v}'_i)]$ ~is the Hamming distance\footnote{We shall use $ d_H $ to denote Hamming distance between two vectors. $\delta(x,y)$ is the kronecker distance and is defined as ~$\delta(x,y) = 1$ ~whenever $x=y$, and $0$ otherwise.}. 
%It is a strong mode if $ p(\textbf{v}) > \sum_{d_H(\textbf{v},\textbf{v}') = 1}p(\textbf{v}') $.
\end{definition}

\begin{definition}[{\bf Perfectly Reconstructible Vectors}]
\label{def-PR}
For an $ \RBM_{n,m}(\theta) $ we define the function $ \up:\{0,1\}^n \to \{0,1\}^m $ that takes a visible vector $ \textbf{v} $ as input and outputs the most likely hidden units vector $ \textbf{h} $ conditioned on $ \textbf{v} $, i.e., $ \up(\textbf{v}) \triangleq \argmax_{\textbf{h}} P(\textbf{h}|\textbf{v},\theta) $. Similarly $ \down(\textbf{h}) \triangleq \argmax_{\textbf{v}} P(\textbf{v}|\textbf{h},\theta) $. A visible units vector $ \textbf{v} $ is said to be \emph{perfectly reconstructible} (\textbf{PR}) if $ \down(\up(\textbf{v})) = \textbf{v} $.
\end{definition}
\noindent
For any set $C$ the cardinality will be denoted by $|C|$.
For an $ \RBM_{n,m}(\theta) $ we define  
\[\prv(n,m,\theta) \triangleq |\{\textbf{v}: \textbf{v} \text{ is \textbf{PR} for } \RBM_{n,m}(\theta)\}| \] 
%For the set $ \textbf{RBM}_{n,m} $ $ \prv(n,m) \triangleq \max_{\theta}\prv(n,m,\theta) $.

\section{Problem Statement}\label{section-problem-statement}

We consider fitting an $ \RBM_{n,m}(\theta) $ to a distribution $ p(\textbf{v}) = \frac{1}{k}\sum_{i=1}^k \delta(\textbf{v} - \textbf{v}_i) $ where $ \delta $ denotes the Dirac Delta function and where for each pair of vectors $ \{\textbf{v}_i,\textbf{v}_j\} $ in $ \{\textbf{v}_r\}_{r=1}^k $, $ d_H(\textbf{v}_i, \textbf{v}_j) \ge 2$. %Furthermore we assume that the $ l_1 $ norm property from [\cite{montufar2015does} Theorem 3.16] holds, i.e., \textbf{for the fitted $ \RBM_{n,m}(\theta) $ the modes are same as perfectly reconstructible vectors}. 
We need to find the smallest $ m^* \in \mathbb{N}$ such that the set $ \textbf{\RBM}_{n,m^*} $ contains an $ \RBM_{n,m^*}(\theta) $ that represents $ p $. We also study the case of a DBM with 2 hiddden layers. We denote a DBM with $ n $ visible units, $ L $ hidden layers with $ m_k $ hidden units in layer $ k $ by $ \RBM_{n,m_1,\ldots,m_L}(\theta) $. We denote the respective set of DBMs by $ \textbf{RBM}_{n,m_1,\ldots,m_L} $. We would like to ask the following question.  
Are there \emph{lean} two layer architectures, $\textbf{\RBM}_{n,m_1,m_2}$ which can model distributions with the same number of modes as that of distributions generated by a one layer architecture $\textbf{\RBM}_{n,m}$ where ~$m\gg m_1, m_2$.
  
\begin{comment} 
\item A key hindrance in answering the above question is the lack of a suitable capacity measure. We ask is there a tractable  measure which could relate the architecture with the number of modes. 
\end{comment}

\subsection{Related Work}
The representational power of Restricted Boltzmann Machines (RBMs) is an ongoing area of study \cite{le2008representational,montufar2011expressive,van2011discriminative,martens2013representational,cueto2010geometry}.It is well known that an RBM with one hidden layer is a universal approximator \cite{le2008representational,montufar2011refinements,montufar2017hierarchical}.
\cite{le2008representational} showed that the set $\textbf{RBM}_{n,m}$ can approximate any input distribution with support set size $ k $ arbitrarily well if following inequality is satisfied. 
\begin{equation}\label{eq:m-lowerbound}
 m \ge k+1
\end{equation}
If we know the number of modes of our input distribution, then we could design our RBM as per Eqn \eqref{eq:m-lowerbound}. Unfortunately the number of modes could be large resulting in a large RBM. 
%In particular, if $ m \ge 2^n + 1 $ then $\textbf{RBM}_{n,m}$ can approximate any input distribution on $ n $ binary variables. The bound was improved by \cite{montufar2011refinements} to $ m \ge 2^{n-1} - 1 $ \textcolor{red}{and later by \cite{montufar2017hierarchical} to $ m \ge \left(\frac{1 + \log n}{1+n}  \right)2^{n+1}-1 $.} 

%\begin{theorem}[Corollary 1 in \cite{montufar2011refinements}]
%\label{thm:known-upper-bound}
%Any distribution on $ \{0, 1\}^n $ can be approximated arbitrarily well by an RBM %with $ \frac{2^n}{2} - 1 $ hidden units.
%\end{theorem}

\begin{figure}[t!]
  \centering
      \includegraphics[width=3in]{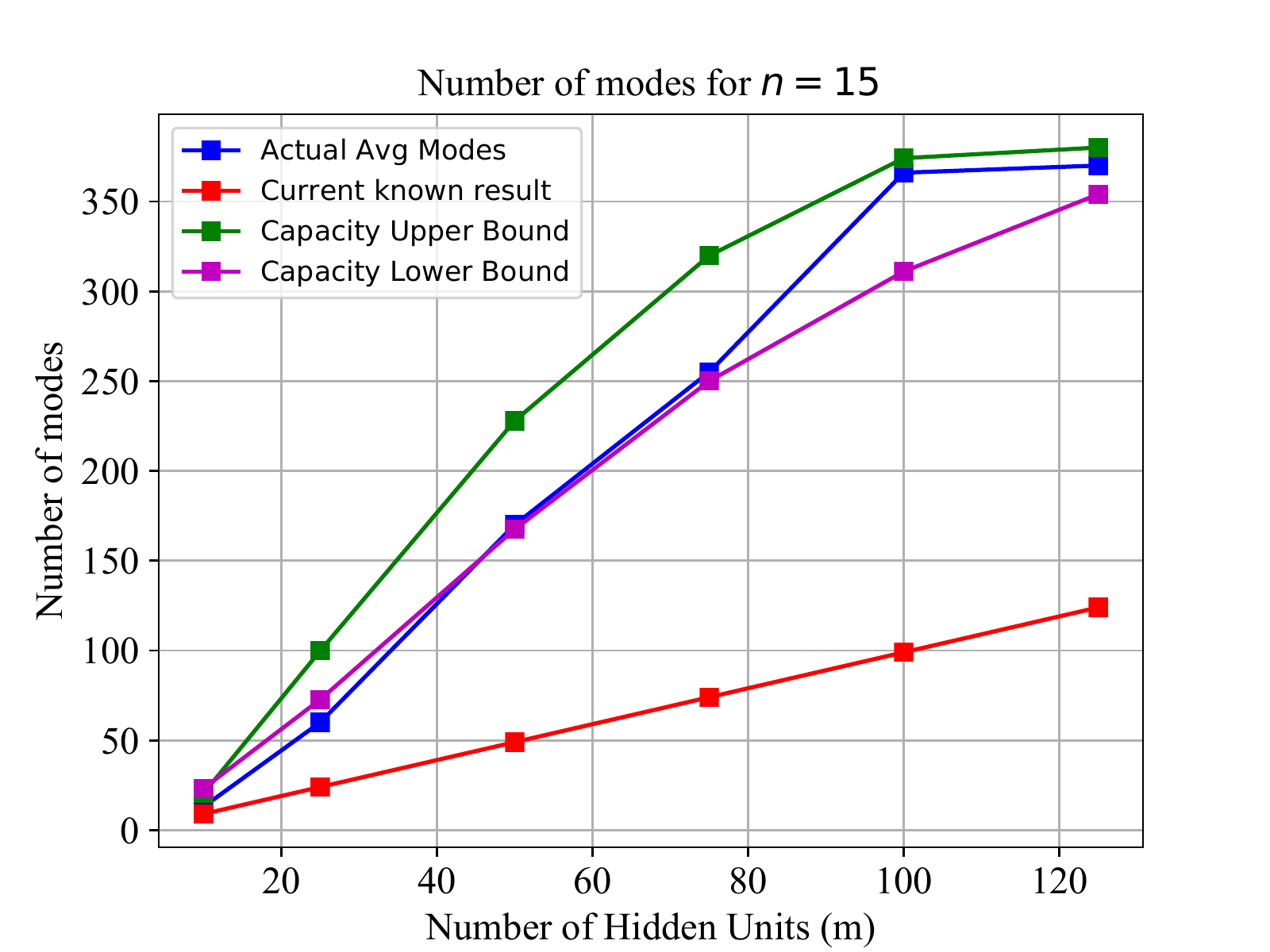}
      \caption{Number of modes attained for different choices of hidden units for $ n = 15 $. Can be seen that the current known result for the number of hidden units required (red graph) is a large over-estimate. The green and purple graphs are estimates given by Theorem \ref{thm-main}. These are closer to the actual number of enumerated modes, given by the blue graph.}
      \label{fig-hid-modes}
\end{figure}

To test the bound in Eqn \eqref{eq:m-lowerbound}, we conducted simulation experiments. We kept $ n = 15 $, $ m \in \{10, 25, 50, 75, 100, 125\} $ and generated random coupling weight matrix whose entries were i.i.d. $ \mathcal{N}(0,1) $ and enumerated all the modes of the generated distribution. We averaged our readings over 100 different weight matrices. The results are shown in Figure \ref{fig-hid-modes}. The results show that the bound gives a highly conservative estimate for $ k $. \emph{For example on average the set $\textbf{\RBM}_{15,50}$ has the capability to represent distributions with 170 modes, instead of only 49 modes}. Thus although the number of modes is an important design criteria, a more practical metric is desirable.

%Large number of hidden units result in more parameters to tune and thus larger training resources. It can also be seen that the number of modes tapers as $ m $ increases. This implies that although for a large $ m $ an RBM is a universal approximator, it would be practically difficult to train if the number of modes in the input distribution is significantly large. 

%The number of perfectly reconstructible vectors bears significance with regards to the memory capacity of the RBM. It indicates the number of input visible vectors that can be \emph{reconstructed} by the RBM. Thus the reconstruction error viz. the average sum of squared differences between the input visible vectors and their reconstructions shows how much the RBM is able to remember the training dataset.

%\input{section4-ppd-to-capacity}
\begin{comment}
\note{We need to mention Bengio as well to bring in the modes}
\cite{montufar2015does} showed that the number of perfectly reconstructible vectors are an indicator of the expressive power of an RBM. However they found that for a general $ \RBM_{n,m}(\theta) $ the problem of computing this was related to the problem of computing covering numbers of a hypercube which is a known hard problem.  
In this section we will motivate a capacity measure which can be interpreted as the average number of perfectly reconstructible vectors.
\end{comment}

\def\bs{{\textbf{s}}}
\def\P{{\mathcal{P}}}
\section{Inherent structures of RBM}\label{section-inherent-structures}
To understand the complex structure in Spin glasses the notion of \emph{Inherent Structures}(IS) was  introduced in \cite{stillinger1982hidden}. 
The IS approach consists of partitioning the configuration space into \emph{valleys}, where each valley consists of configurations in the vicinity of a local minimum. The number of such valleys can thus be indicative of \emph{Complexity} of the system.

In this section we recall the IS approach in a general setting to motivate a suitable capacity measure.  
Consider a system governed by the probability model
\begin{equation}\label{eq:model}
P(\textbf{s}|W) = \frac{e^{- \frac{1}{T} E(\textbf{s}|W)}}{Z(W)} , Z(W) = \sum_{\textbf{s}}e^{- \frac{1}{T} E(\textbf{s})}
\end{equation} 
where ~$E: \{0,1\}^N \rightarrow \RR$ is an energy function defined over ~$N$ 
dimensional binary vectors with parameter ~$W$. 

\begin{definition}[{\bf One-flip Stable States}]{\cite{stein2013spin}}
\label{def:one-flip}
For an Energy function $E$ a configuration, $\bs^*$ is called a local minimum, also called One flip stable state, if $ \forall \textbf{s} \in \{s: d_H(\bs, \bs^*) = 1\},  E(\bs) - E(\bs^*) > 0 $ (equivalently $ P(\bs) < P(\bs^*) $). %Here $d_H(\bs, \bs^*)= \sum_{i=1}^N [1-\delta(\bs_i,\bs^*_i)]$ ~is the Hamming distance\footnote{$\delta(x,y)$ is the kronecker distance and is defined as ~$\delta(x,y) = 1$ ~whenever $x=y$, and $0$ otherwise.} between ~$\bs$ and ~$\bs^*$.
 
\end{definition}
 
For every one-flip stable state $ \bs^* $ we define the set $ OF(\bs^*) = \{\bs| \quad d_H(\bs,\bs^*) \le 1 \} $. Let $\{\P_1,\ldots, \P_K\}$ ~form a partition of the configuration space where 
each ~$\P_a = OF(\bs_a)$ corresponds to the local minimum~$\bs_a$ ~and~$K$ is the total number of valleys \footnote{Here we assume that temperature parameter $ T $ is small so that all states with Hamming distance $ > 1 $ from a one-flip stable state will have negligible contribution to partition function.}. The logarithm of the partition function
$$\log Z(W) = \log \sum_{\bs} e ^{-\frac{1}{T}E(\bs)} =\log \sum_{a=1}^K Z_a(W)$$
where $ Z_a(W) = \sum_{\bs \in P_a} e^{-\frac{1}{T}E(\bs)}$. Now, for any $p$ in a 
$K$~dimensional probability simplex, using the non-negativity of KL divergence,
 it is straightforward to show that
\begin{equation}\label{eq:rel} 
H(p) + \sum_{a=1}^K p_a \log Z_a(W) \le \log Z(W) 
\end{equation}
where ~$H(p) = -\sum_{a=1}^K p_a \log p_a $~is the entropy of $p$. 
Equality holds whenever $p_a^* = \frac{Z_a}{Z}, \forall a \in \{1,\ldots,K\}$.
One could construct $\log Z$ from ~$\log Z_a$ if one had access to $p^*$, and 
knew $K^*$ which is defined as  $H(p^*) = log K^*$. 
\begin{equation}\label{eq:relone} 
\log K^* + \sum_{a=1}^K p^*_a \log Z_a(W) = \log Z(W) 
\end{equation}
From the properties of entropy function one could write  
\begin{equation}\label{eq:id}
1 \le K^* \le K \le 2^{N} 
\end{equation}  
where the lower bound on ~$K^*$ is attained at ~$H(p^*)=0$~and is realized 
when the  Energy surface has only one local minimum, a very un-interesting case.
The upper bound on $K^*$ is attained at ~$p^* = \frac{1}{K}$, which happens only when all valleys are considered similar. 
Since Number of states can be at most $2^N$, the last upper bound holds and 
Thus $\frac{1}{N}\log_2 K$, can be viewed as a measure of Complexity, of the  energy surface. 
One could put a suitable prior distribution over the parameters $W$ and evaluate the complexity averaged over the prior, motivating the following definition.
\begin{definition}({\bf Complexity})
The  \emph{Complexity} of the model described in Eqn \eqref{eq:model} is given by 
$$ \frac{1}{N} \EE_{W}\log_2 K, \quad W \sim {\textbf P}, $$
where $K$ is the number of One-Flip stable states for Energy function defined with parameter $W$ and  ${\textbf P}$ is a prior distribution over ~$W$.
\end{definition} 

For Ising models,  
\emph{Complexity} has been estimated in the large $ N $ limit \cite{bray1980metastable,tanaka1980analytic} by methods such as \emph{Replica technique}. However, extending their methods to RBMs for a finite size $ N $ is not straightforward. 

It has been shown (see e.g. \cite{parisi1995mean}) that IS decomposition gives a very accurate picture of energy landscape of Ising models at  Temperature, $T =0$. But, for $T > 0$, one needs to take into account both the 
Valley structure and the energy landscape of the free energy 
\cite{biroli2000inherent}. Obtaining accurate estimates of Complexity 
is an active area of study, for a recent review see \cite{auffinger2013random}.

Our goal is to apply the aforementioned IS decomposition to RBMs. %For the distributions considered in Section \ref{section-problem-statement}, the support set (modes) is same as the set of perfectly reconstructible vectors [\cite{montufar2015does} Theorem 1.6], ie,
%\[\prv(n,m,\theta) = \left\lbrace \textbf{v} | p(\textbf{v}) > 0\right\rbrace\]
We now show the equivalence between these perfectly reconstructible vectors and one-flip stable states for an RBM. The IS decomposition then allows us to define the measure of capacity in terms of the modes of input distribution.
\begin{restatable}{lemma}{lemequiv}
\label{lem:equiv-prv-ofs}
A vector $ \textbf{v} $ is perfectly reconstructible for an $ \RBM_{n,m}(\theta) $ $ \iff $ the state $ \{\textbf{v}, \up(\textbf{v})\} $ is one-flip stable.
\end{restatable}
\begin{proof}
See Supplementary material.
\end{proof}
Thus we see that there is a one-one equivalence between perfectly reconstructible vectors and the one-flip stable states for a single layer $ \RBM_{n,m}(\theta) $. 

\noindent
{\bf Relationship between the modes of ~$p(\textbf{v})$~and~$p(\textbf{v},\textbf{h})$}
In this section we discuss the relationship between the modes of the marginal distribution, $p(\textbf{v})$~and the joint distributuon ~$p(\textbf{v},\textbf{h})$.
We make a mild assumption on one-flip stable state.
\begin{itemize}
\item[~\bf{A1}]
For a single layer RBM, given a visible vector $ \textbf{v} $, vector $ \textbf{h}^* = \up(\textbf{v}) $ is unique. 
\end{itemize}
If the weights are given small random perturbation, then Assumption 1 holds with probability one. However it does not hold true for an $ L \geq 2$ layer $ \RBM_{n,m_1,...,m_L}(\theta) $. We denote $ \{\textbf{h}_l \in \{0,1\}^{m_l}\}_{l=1}^L $ to be hidden vectors, $ \textbf{v} \in \{0,1\}^n $ to be visible vector and define the set
\[\mathcal{H}(\textbf{v}) \triangleq \{\{\textbf{h}_l\}_{l=1}^L |  (\textbf{v},\{\textbf{h}_l\}_{l=1}^L) \mbox{ is one-flip stable state} \}\]
It can be seen that $ |\mathcal{H}(\textbf{v})| $ can be more than one. For input distributions considered in Section \ref{section-problem-statement}, the modes of joint distribution $ p(\textbf{v},\textbf{h}_1,\ldots, \textbf{h}_L) $ with distinct $ \textbf{v} $ are atleast as many as modes of marginal distribution p(\textbf{v}). A formal statement with proof is given in supplementary material.

As discussed, for $L \geq 1$, the modes of the marginal distribution could be smaller than modes of the joint distribution. However, \cite{montufar2015does} [Theorem 1.6] gave precise conditions under which the number of modes for marginal and joint distributions are same for a single layer network. We suspect that a similar argument holds for $L > 1$. For the rest of the paper we will assume that the modes of joint distribution are same as those of $p(\textbf{v})$.

%Moreover for the distributions considered in Section \ref{section-problem-statement}
%$ \RBM_{n,m_1,...,m_L}(\theta) $ shall have modes at $ \{\textbf{v}\} $ for which $ |\mathcal{H}(\textbf{v})| \ge 1 $. This is because if $ |\mathcal{H}(\textbf{v})| \ge 1, \implies \exists \{\textbf{h}_l\}_{l=1}^L \text{ such that } (\textbf{v},\{\textbf{h}_l\}_{l=1}^L) \mbox{ is one-flip stable}\implies P(\textbf{v},\{\textbf{h}_l\}_{l=1}^L) > 0 \implies P(\textbf{v}) > 0 $. Conversely if $ v $ is a mode $ \implies P(\textbf{v}) > 0\implies \exists \{\textbf{h}_l\}_{l=1}^L = \argmax_{\{\textbf{h}_l\}} P(\textbf{v},\{\textbf{h}_l\}_{l=1}^L) \implies (\textbf{v},\{\textbf{h}_l\}_{l=1}^L)$ is one-flip stable since all neighbours of $ \textbf{v} $ have zero probability. 
%, i.e., \[ \{\textbf{v}: \exists \{\textbf{h}_l\}_{l=1}^L \text{ such that } (\textbf{v},\{\textbf{h}_l\}_{l=1}^L) \mbox{ is one-flip stable state} \} \]

Armed with these observations we are now ready to define a measure which relates the architecture of a DBM and the expected number of  such modes under a prior distribution on the model parameters. More formally,
\begin{definition}[{\bf Inherent Structure Capacity}]
\label{def:capacity}
For an $L$~layered DBM with $m_1,\ldots,m_L$~hidden units and ~$n$~visible units we define the \emph{Inherent Structure Capacity} ($ \ISC $), denoted by  $C(n,m_1,\ldots,m_L)$, to be the logarithm (divided by $ n $) of the expected number of modes of all possible distributions  generated over the visible units by the DBM.
\[\mathcal{C}(n,m_1,\ldots,m_L) = \frac{1}{n}\log_2\EE_{\theta}\left[\left\lvert\{\textbf{v}: |\mathcal{H}(\textbf{v})| \ge 1 \}\right\rvert \right]\] 

\end{definition}

We note that for the single layer case this definition reduces to $ \frac{1}{n}\log_2\EE_{\theta}\left[\prv(n,m,\theta) \right] $. 
$ \ISC $ as a measure would be useful in identifying DBM architectures which can model modes of an input distribution defined over the 
visible units.
 
This measure serves as a recipe for fitting DBMs. Suppose we know that the input distribution has  ~$k$~modes then one could find a suitable DBM architecture, i.e. $m_1,\ldots,m_L$  by the following criterion
\begin{equation}\label{eq:designchoice}
\frac{1}{n}\log_2 k \leq \mathcal{C}(n, m_1,\ldots,m_L)
\end{equation}
Once the architecture has been identified one can then use a standard learning algorithm to learn parameters to fit a given distribution. 

In the following 
sections we investigate the computation of $ \ISC $ and their applications to single and two layer networks, i.e. $L=1$ and ~$L=2$. To keep the exposition simple we assume the bias parameters to be zero \footnote{Analysis can be extended to non-zero biases in straightforward manner.}. We also assume that the coupling weights are distributed as per mean zero Gaussian, i.e., $ \forall i,j, w_{ij} \sim \mathcal{N}(0,\sigma^2) $.

\section{Computing capacity of $ \RBM_{n,m} $ and need for more layers}
In this section we discuss the computation of $ \ISC $ for a single layer RBM. In absence of a definitive proof we conjecture that $ \ISC $ is intractable just like the Complexity measure in Spin glasses. 
The problem of computing Complexity  has been addressed in the Statistical Mechanics community using the Replica method \cite{roberts1981metastable,de1980white} which yields reasonable estimates. However the applicability of Replica trick to Multi-layer DBMs is not clear.
In this section we develop an alternative method for estimating $ \ISC $.
\subsection{Computing $ \ISC $ of $\RBM_{n,m}$}
 For any arbitrary vector $ \textbf{v} \in \{0,1\}^n $ we compute $ \EE\left[ \O1_{[\textbf{v} \text{ is PR.}]} \right]$ where $ \O1 $ is the indicator random variable and expectation is over the model parameters $ \theta $ with prior as stated in Section \ref{section-inherent-structures}. We then sum this over all $ 2^n $ vectors, i.e., $ \sum_{\textbf{v}}\EE\left[ \O1_{[\textbf{v} \text{ is PR.}]}\right] $.
Before stating our main theorem we state a few Lemmas. 

\vspace{1mm}
\begin{restatable}{lemma}{lemindep}
\label{lem-indep}
For the set $ \textbf{RBM}_{n,m} $, if a given vector $ \textbf{v} $ has $ r (\ge 1)$ ones, $ \textbf{h} = \up(\textbf{v}) $ has $ l $ ones and $ l \gg 1 $,then \footnote{Here $ l \gg 1 $ means $ l $ is atleast 50 hidden units, which according to us is a reasonable assumption.} for $r > 1$,

$$ \EE\left[\O1_{[\textbf{v} \text{ is PR.}]} \right]  
 \le \left[ \frac{1}{2} - \frac{1}{2}\erf \left( - \sqrt{\frac{l}{\pi r - 2}}\right) \right]^r \left(\frac{1}{2}\right)^{n-r}.   $$
For $r =1,$ the expression
$ \EE\left[\O1_{[\textbf{v} \text{ is PR.}]} \right]$ equates to ${\left(\frac{1}{2}\right)}^{n-1}$.  
where $ \erf(x) = \frac{1}{\sqrt{\pi}}\int_{-x}^{x} e^{-t^2}dt $
\end{restatable}
\begin{proof}
See Supplementary Material.
\end{proof}
\noindent
For $ r (> 1) $ ones in $ \textbf{v} $ and $ l $ ones in $ \textbf{h} = \up(\textbf{v}) $ the problem of computing $ \left \lbrace P[[\down(\textbf{h})]_i  = 1]\right \rbrace_{i=1}^r $ can be reformulated in terms of matrix row and column sums, viz, given $ W \in \RR^{r\times l} $ where all entries $ w_{ij} \sim \mathcal{N}(0,\sigma^2)$ are i.i.d. and given that all the column sums $ \left \lbrace C_j = \sum_{i = 1}^r w_{ij} > 0 \right \rbrace_{j=1}^l $, to compute the probability that all the row sums are positive, i.e., $ \left \lbrace R_i = \sum_{j = 1}^l w_{ij} > 0 \right \rbrace_{i=1}^r $. Conditioned on the fact $ \left \lbrace C_j > 0 \right \rbrace_{j=1}^l $ the random variables $ \left \lbrace R_i \right \rbrace_{i=1}^r $ are negatively correlated. This gives us an upper bound mentioned in Lemma~\ref{lem-indep}. We now get a lower bound for the estimate.

\vspace{1mm}
\begin{restatable}{lemma}{lemtildconst}
\label{lem-tild-constants}
For the set $ \textbf{RBM}_{n,m} $, if $ \textbf{v} $ has $ r (>1)$ ones, $ \textbf{h} = \up(\textbf{v}) $ has $ l $ ones, then $ \exists \mu_c, \tilde{\mu}_c, \sigma_c, \tilde{\sigma}_c \in \RR_+$ such that conditioned on $ \{R_t > 0\}_{t=1}^{i-1}, C_j > 0 $, the moments of posterior distribution of $ w_{ij}$ is given by 
\begin{eqnarray*}
\EE\left[w_{ij} | \{R_t > 0\}_{t=1}^{i-1}, C_j > 0 \right]&=& (\tilde{\mu}_{c} - \mu_c)\frac{\sigma^2}{\sigma_c^2}  \\
\text{Var}\left[w_{ij} | \{R_t > 0\}_{t=1}^{i-1}, C_j > 0 \right] &=& \tilde{\sigma}_{c}^2\left(\frac{\sigma^2}{\sigma_c^2} \right)^2 + \sigma^2 \beta 
\end{eqnarray*}
where $ \beta = \left(1 - \frac{\sigma^2}{\sigma_c^2}\right) $
%where $\mu_c = (i-1)\sigma \sqrt{\frac{2}{\pi l}}, \sigma_c^2 = (i-1)\sigma^2 \left(1 -\frac{2}{\pi l}\right) + (r-i+1)\sigma^2,
%\tilde{\mu}_c = \mu_c + \sigma_c \frac{\phi}{Z}, \tilde{\sigma}_c^2 = \sigma_c^2 \left[ 1 - \frac{\mu_c \phi}{\sigma_c Z} - \frac{\phi^2}{Z^2} \right], Z = \frac{1}{2} - \frac{1}{2}\erf\left(-\frac{\mu_c}{\sigma_c\sqrt{2}}\right) $ and
%$ \phi = \frac{1}{\sqrt{2\pi}}e^{\left(-\frac{\mu_c^2}{2\sigma_c^2}\right)} $.
\end{restatable}
\begin{proof}
See Supplementary Material.
\end{proof}
\noindent
Lemma \ref{lem-indep} gives an upper bound $ U(n,m) $ on expected number of \emph{PR} vectors while Lemma \ref{lem-tild-constants} gives us a posterior distribution on $ w_{ij} $ after taking into account the conditional correlation between $ \{R_i\}_{i=1}^r $. This eventually results in a lower bound $ L(n,m) $.
Thus even though a closed-form expression for $ \ISC $ is difficult, we obtain bounds on it as the following theorem states.
\begin{restatable}{theorem}{thmcapacity}{(\bf $ \ISC$ of $\RBM_{n,m} $)}
%\begin{theorem}{(\bf $ \RBM_{n,m} $ Capacity)}
\label{thm-main}
There exist non-trivial functions $ L(n,m), U(n,m): \mathds{Z}\times\mathds{Z}\rightarrow \mathds{R}_+ $ such that $ \ISC $ of the set $ \textbf{\RBM}_{n,m} $ obeys the following inequality.
\[ \frac{1}{n}\log_2 (L(n,m)) \le \C(n,m) \le \frac{1}{n} \log_2 (U(n,m)) \] % = \left[\frac{3}{2} - \frac{1}{2}\exp \left( - 0.846\sqrt{\frac{m}{n}}\right) \right]^n  \]
%such that
%\[ \lim_{m \to \infty} L_{n,m} = \lim_{m \to \infty} \C(n,m)  = \lim_{m \to \infty} U_{n,m} = (1.5)^n \]
%where 
%\begin{equation*}
%U(n,m) = \sum_{r = 1}^n \binom{n}{r} \sum_{l=1}^m \binom{m}{l} \left(\frac{1}{2}\right)^m \left[ \frac{1}{2} - \frac{1}{2}\erf \left( - \sqrt{\frac{l}{\pi r - 2}}\right) \right]^r \left(\frac{1}{2}\right)^{n-r}
%\end{equation*}
%\begin{equation*}
%\begin{split}
%L(n,m) &= \sum_{r=1}^n \binom{n}{r}\sum_{l=1}^m \binom{m}{l} \left(\frac{1}{2}\right)^m \\
%&\left \lbrace\prod_{i = 1}^{r}\left[ \frac{1}{2} - \frac{1}{2}\erf \left(-\frac{\tilde{\mu}_{i}(r,l)\sqrt{\frac{l}{2}}}{\tilde{\sigma}_{i}(r,l)}\right) \right]\right \rbrace \left(\frac{1}{2}\right)^{n-r}
%\end{split}
%\end{equation*}
%and $ \tilde{\mu}_{i}(r,l), \tilde{\sigma}_{i}(r,l)  $ are as defined in Lemma \ref{lem-tild-constants}.
\end{restatable}
\begin{proof}
See Supplementary material.
\end{proof}

\subsection{Need for more hidden layers}
Theorem \ref{thm-main} establishes the lower and upper bounds for $ \ISC $. A 
direct corollary of the theorem establishes that $C(n,m)$ approaches a limit as ~$m$ 
increases.    
\begin{restatable}{corollary}{corlimit}{\bf (Large $ m $ limit)}
\label{cor:m-inf}
For the set $ \textbf{\RBM}_{n,m} $,  
$\lim_{m \to \infty} \C(n,m)  = \log_2 1.5 = 0.585$
where $ C(n,m)$ is defined in Theorem~\ref{thm-main}.
\end{restatable}
\begin{proof} In the Supplementary material we show that 
$ \lim_{m \to \infty}\frac{1}{n}\log_2 L(n,m) = \lim_{m \to \infty} \frac{1}{n}\log_2 U(n,m)= \log_2 1.5$. Then claim follows from squeeze theorem\footnote{\url{http://mathonline.wikidot.com/the-squeeze-theorem-for-convergent-sequences}}. 
 \end{proof}
Empirically we observe that this saturation limit is achieved when $ m > 20n $ (see Figure \ref{fig-comparison}).
%\begin{figure}[h]
%  \centering
%      \includegraphics[width=3in]{capacity-comparison-known}
%      \caption{Capacity estimates for RBM with $ n = 50 $. The values are plotted on a log-scale. It can be seen that the estimates are much above the current known bounds.}
%      \label{fig-capacity-known-comp}
%\end{figure}
\noindent
Here we discuss the implications of the results derived in the previous subsection.
\begin{enumerate}
%\item
%We have shown that the capacity measure $ \mathcal{C}(n,m) $ can be calculated without resorting to the classical replica approach that is valid only for large $ n $ regime and that cannot be extended to the bipartite graph structure of RBMs in a straightforward manner. Although \cite{montufar2015does} stated that the number of perfectly reconstructible vectors can be a useful measure to assess the RBMs capacity, the measure was not computationally tractable in the exact sense. Instead by assuming a prior on the model parameters we are able to devise a method that calculates this for any general $ \textbf{\RBM}_{n,m} $
\item
We plotted the actual expected modes attained and the $ \ISC $ estimates derived from Theorem \ref{thm-main} for $ n = 15 $ and varying number of hidden units (Figure \ref{fig-hid-modes}). We can see that even a small number of hidden units admits a large $ \ISC $ and the current known bound given in Equation \ref{eq:m-lowerbound} is not necessary. This shows that for a large class of distributions we give a more practical estimate of number of hidden units required than the current state of the art.  
\item
%\cite{le2008representational} and \cite{montufar2011refinements} showed that an RBM with sufficiently large number of hidden units can approximate any distribution on $ \{0,1\}^n $. In light of this 
The upper bound on the $ \ISC $ estimated above seems surprising at first sight since it seems to contradict the well established fact that RBMs are universal approximators \cite{freund1992unsupervised,le2008representational}.  However, one should note that the bound is in expected sense which means that in the family $ \textbf{\RBM}_{n,(m\to \infty)} $ many RBMs shall have modes close to or less than $ (1.5)^n $. For the class of input distributions for which number of modes $ k \gg (1.5)^n $ training an $ \RBM_{n,m}(\theta) $ to represent these might be difficult. \emph{The need for multi-layering arises in such conditions}.
\item
Corollary \ref{cor:m-inf} shows for a large enough $ m $ the bounds become tight and the expression is exact. We also show this through simulations in Section \ref{section-experiments}.
\end{enumerate}
\begin{remark}
When $ n, m \to \infty $ we can approximate $ U(n,m) $ by the following relatively simple expression that we can use to conduct further analysis. %which we use to carry out our analysis and experiments. We approximate $ l \approx \frac{m}{2} $ since each hidden unit will be on with probability $ \frac{1}{2} $ so approximately $ \frac{m}{2} $ hidden units shall be on. 
%We further replace the $ r $ inside the $ \erf $ function by $ \frac{n}{2} $ assuming that the terms for which $ r = \frac{n}{2} $ shall dominate the expression. Then
\begin{align}\label{eq-approx}
\begin{split}
U(n,m) &\approx \left[1 - \frac{1}{2}\erf \left( - \sqrt{\frac{m}{\pi n - 4}}\right) \right]^n %\\ &\approx \left[\frac{3}{2} - \frac{1}{2}\exp \left( - 1.5\sqrt{\frac{m}{\pi n}}\right) \right]^n \left[1+ \smallO(1)\right]
\end{split}
\end{align}
%where the $ \erf $ function has been approximated by an exponential. The factor of $ 1.5 $ in the exponent was empirically found by comparing the two functions.
\end{remark}

\begin{comment}
{\bf Relationship with \cite{montufar2015does}} From the discussion in Section \ref{section-problem-statement} we see that the capacity is the expected number of perfectly reconstructible vectors. In contrast \cite{montufar2015does} propose the measure $ \prv(n,m) = \max_{\theta} [\prv(n,m,\theta)] $. They reduce the computation of this measure to the problem of computing covering numbers of a hypercube. Unfortunately this is computationally intractable for a practically sized RBM. They show that $ \prv(n,m) = \min(2^m, 2^{n-1}) $. Instead we show that when $ m \ll n $, the capacity $ \C(m,n) = 1.5^m $. Thus we give a lower bound on $ \prv(n,m) $. Moreover estimating the capacity as defined here is computationally tractable and can be used to assess or compare different architectures. 
\end{comment}
\section{$ \ISC $ of two-layer $ \RBM $ architecture}
To study the effect of adding layers, we consider the family $ \textbf{\RBM}_{n,m_1, m_2} $. As stated in Section \ref{section-inherent-structures}, adapting analysis for single layer RBMs to multi-layer RBMs is not straightforward. 
In this section we discuss the computation of $ \ISC $ and study its application to design RBMs.

\subsection{Computing the capacity of 2 layer RBM}
\label{section-results-implications}
We observe that an $ \textbf{\RBM}_{n,m_1, m_2} $ shares the same bipartite structure as a single layer $ \textbf{\RBM}_{n+m_2,m_1} $ (Figure \ref{fig:dbm-rbm}). This enables us to extend our single layer result to two layers. We introduce a threshold quantity $ \gamma = 0.05 $. This value was obtained by simulating the asymtotics of $ f(x) = 1 -0.5\erf(-\sqrt{\frac{x}{\pi}}) $.

%\begin{figure}[h]
%  \centering
%      \includegraphics[height=2in]{dbm}
%      \caption{A cartoon description of a DBM with two hidden layers. The number of units in each layers is marked.}
%      \label{fig-dbm}
%\end{figure}

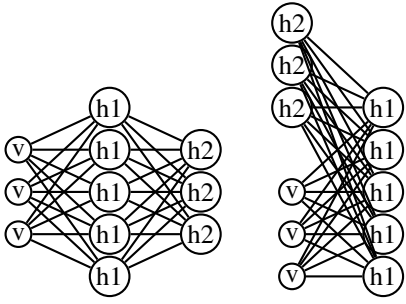
\begin{figure}
\centering

\begin{tikzpicture}[scale=0.8]
  \foreach[evaluate=\x as \r using 0.7+(\x-1)*.7] \x in {1,...,3}
    \node[vertex][inner sep=1pt,minimum size=3pt] (\x) at (0,\r) {v};
  \foreach[evaluate=\y as \r using (\y-4)*.7] \y in {4,...,8}
      \node[vertex][inner sep=1pt,minimum size=3pt] (\y) at (1.5,\r) {h1};
  \foreach[evaluate=\z as \r using 0.7+(\z-9)*.7] \z in {9,...,11}
      \node[vertex][inner sep=1pt,minimum size=3pt] (\z) at (3,\r) {h2};      
  \foreach \x in {1,...,3}
  {
    \foreach \y in {4,...,8}
    	\draw[edge] (\x) to (\y);
  }
  \foreach \x in {4,...,8}
  {
    \foreach \y in {9,...,11}
    	\draw[edge] (\x) to (\y);
  }
  
  \foreach[evaluate=\x as \r using (\x-12)*.7] \x in {12,...,14}
      \node[vertex][inner sep=1pt,minimum size=3pt] (\x) at (4.5,\r) {v};
    \foreach[evaluate=\z as \r using 2.8+(\z-20)*.7] \z in {20,...,22}
        \node[vertex][inner sep=1pt,minimum size=3pt] (\z) at (4.5,\r) {h2};      
    \foreach[evaluate=\y as \r using (\y-15)*.7] \y in {15,...,19}
        \node[vertex][inner sep=1pt,minimum size=3pt] (\y) at (6,\r) {h1};

    \foreach \x in {12,...,14}
    {
      \foreach \y in {15,...,19}
      	\draw[edge] (\x) to (\y);
    }
    \foreach \x in {20,...,22}
    {
      \foreach \y in {15,...,19}
      	\draw[edge] (\x) to (\y);
    }
    
\end{tikzpicture}
\caption{Two Layer $ \RBM_{n,m_1,m_2} $ shares same bipartite graph structure as single layer $ \RBM_{n+m_2,m_1} $}
\label{fig:dbm-rbm}
\end{figure}

\begin{restatable}{theorem}{thmdbmcap}{\bf (\ISC~ of  $ \RBM_{n,m_1,m_2} $ )}
\label{thm-dbn}
For an $ \textbf{\RBM}_{n,m_1,m_2} $ ($ n,m_1 > 0$ and $m_2 \ge 0 $), if we denote $u = \max(m_1, n+m_2), l = \min(m_1, n+m_2)$,  then 
$$ \C(n,m_1, m_2) \le \frac{1}{n}\log_2 S $$
whenever $ S < \gamma2^n,\; 
S= \left[1 - \frac{1}{2}\erf \left( - \sqrt{\frac{u}{\pi l - 4}}\right) \right]^{l}$ 
\end{restatable}
\begin{proof}
See Supplementary material.
\end{proof}

Theorem \ref{thm-dbn} gives a general formula from which different regimes can be derived by varying $ m_1,m_2 $. We will use this theorem to understand the 
design of multi-layer RBMs. 

In the previous section we saw that in a single layer RBM, irrespective of 
number of hidden units, $ \ISC $, achieves a limiting value of $0.585$. The theorem 
will be useful to quantitatively show that $ \ISC $ can indeed be improved if 
we consider layering. 
For an $ \textbf{\RBM}_{n,m_1,m_2} $ ($ n,m_1 > 0$ and $m_2 \ge 0 $), we denote $ \alpha_1 = \frac{m_1}{n}, \alpha_2 = \frac{m_2}{n} $. We say that a layer with $ m $ hidden units is \emph{narrow} if $ m < \gamma $ and it is \emph{wide} if $ m > \frac{1}{\gamma} $. 

\begin{table*}[t]
\caption{$ \ISC $ Values for different $ \alpha_1 = \frac{m_1}{n}, \alpha_2 = \frac{m_2}{n} $. $ \gamma = 0.05 $ (obtained by simulating the asymtotics of $ f(x) = 1 -0.5\erf(-\sqrt{\frac{x}{\pi}}) $).}
\label{table-capacity}
\begin{center}
% \begin{tabular}{|p{0.75cm}|p{3cm}|p{4cm}|} 
\begin{tabular}{p{0.15\linewidth}p{0.22\linewidth}p{0.57\linewidth}}
 %\begin{tabular}{|c|c|c|} 
 \hline
 Regime & $ \ISC $ & Implications \\ [0.5ex] 
 \hline\hline
 $ \alpha_1 > \frac{1}{\gamma}, \alpha_2 < \gamma $ & $ (1+\alpha_2)\log_2(1.5)$ & $ \ISC $ determined only by $ \alpha_2 $. For a single layer RBM ($ \alpha_2 = 0 $), further increase in hidden units not effective, multi-layering recommended.\\ 
 %\hline
 % $ \alpha_1 < \gamma$ & $ \alpha_1\ln (1.5)$ & Capacity determined only by $ \alpha_1 $, increasing $ \alpha_2 $ does not help. \\ 
 \hline
  $ \alpha_1(1 + \alpha_2) = c$, where $c \ge 1$ & $\min[1,\sqrt{c}\log_2(1.29)] $ ($ \alpha_1^* = \sqrt{c}) $. & Given a budget of $ cn^2 $ parameters, this is the maximum $ \ISC $ achievable with optimal choice of $ \alpha_1 $. \\ 
  \hline  
  $ \alpha_1(1 + \alpha_2) = c$, where $c < 1$ & $c\log_2\left[1 - \frac{1}{2}\erf \left( - \sqrt{\frac{1}{\pi c}}\right) \right] $ ($ \alpha_1^* =c $). & If total number of parameters $ < n^2 $, then multi-layering does not help. \\ 
   \hline  
\end{tabular}
\end{center}

\end{table*}

\begin{restatable}{corollary}{corhonemtwo}{(\bf Layer 1 Wide, Layer 2 Narrow)}
%\begin{corollary}{(\bf )}
\label{cor-high-alpha1-med-alpha2}
For an $ \textbf{\RBM}_{n,m_1,m_2} $ ($ n,m_1 > 0$ and $m_2 \ge 0 $), if $ \alpha_1 = \frac{m_1}{n} > \frac{1}{\gamma}$ and $ \alpha_2 = \frac{m_2}{n} < \gamma $ then 
\[ \C(n,m_1, m_2) \le (1+\alpha_2)\log_2 (1.5) \]
\end{restatable}
%\end{corollary}
\begin{proof}
See Supplementary material.
\end{proof}
The Corollary shows that for a RBM with a wide first layer and a narrow second layer, the upper bound on $ \ISC $ increases linearly with the number of units in second layer.

\subsection{$ \RBM_{n,m_1,m_2} $ design under budget on parameters}

We extend the result obtained in previous section to consider a real scenario wherein we have a budget on the maximum number of parameters that we can use and we have to design a two-layered DBM given this constraint. For a given input distribution with $ k $ modes, the DBM should have $\mathcal{C}(n,m_1,m_2) > \frac{1}{n}\log_2 k $.
\begin{restatable}{corollary}{corfb}{(\bf Fixed budget on parameters)}
\label{cor-fixed-budget}
For an $ \textbf{\RBM}_{n,m_1,m_2} $ ($ n,m_1 > 0$ and $m_2 \ge 0 $), if there is a budget of $ cn^2 $ on the total number of parameters, i.e, $ \alpha_1(1 + \alpha_2) = c$ then the maximum possible $ \ISC $, $ \max_{\alpha_1,\alpha_2} \C(n,\alpha_1, \alpha_2) \le \tilde{U}(n,\alpha_1^*, \alpha_2^*) $ where
\[\tilde{U}(n,\alpha_1^*, \alpha_2^*) = \begin{cases}
\min(1, \sqrt{c}\log_2 (1.29)) & \text{ if }  c \ge 1 \\
c\log_2\left[1 - \frac{1}{2}\erf \left( - \sqrt{\frac{1}{\pi c}}\right) \right] & \text{ if } c < 1
\end{cases}\]
\end{restatable}
\begin{proof}
See Supplementary material.
\end{proof}
Corollary \ref{cor-fixed-budget} can be used to determine the optimal allocation of hidden units to the two layers if there is a budget on the number of parameters to be used due to computational power or time constraints. It says that if $ c \ge 1 $, then for optimality $ \alpha_1 = 1 + \alpha_2 $ and if $ c < 1 $, then $ \alpha_2 = 0 $ which means that all hidden units should be added to layer 1. The following corollary highlights the existence of a two layer architecture $ \textbf{\RBM}_{n,m_1,m_2} $ that has $ \ISC $ equal to $0.585$, the saturation limit for single layer RBMs. 
\begin{corollary} %{(\bf Lean two layer DBM)}
\label{cor:single-double-comparison}
There exists a two layer architecture $ \RBM_{n,m_1,m_2} $ with $ \Theta(n^2) $ parameters such that
\[ \tilde{U}(n,\alpha_1,\alpha_2) = \log_2 1.5 \]
where $m_1 = \alpha_1 n, m_2 = \alpha_2 n, \alpha_1 =1.6$~and~$\alpha_2=0.6$ 
\end{corollary}
\begin{proof}
%From Corollary \ref{cor:m-inf} we get $ \mathcal{C}(n,m) \le \log_2 1.5 $. 
In Corollary~\ref{cor-fixed-budget} if we put $ \tilde{U}(n,\alpha_1^*, \alpha_2^*) = \log_2(1.5) $, we get $ \alpha_1^* = \sqrt{c} = \frac{\log_2(1.5)}{\log_2(1.29)} = 1.6, \alpha_2^* = \alpha_1^* - 1 = 0.6$. Number of parameters for such an RBM is  $ \alpha_1(1 + \alpha_2)n^2  = \Theta(n^2)$.
\end{proof} 
The number of parameters for any single layer RBM is $nm$ where $m$ is number of hidden units. The above corollary gives an important insight: one can construct a two layer RBM with $\Theta(n^2)$ parameters that has the same $ \ISC $ as a single layer 
RBM with infinitely many hidden units. Ofcourse this is true only if the upper-bound $\tilde{U}$ ~is close to $\C$. This suggests that \emph{lean} 2 layer networks with \emph{order of magnitude less number of parameters can achieve the same $ \ISC $ as that of a single layer RBM.} %This supports the known fact that deep narrow Boltzmann Machines are also universal approximators \cite{montufar2014deep}.

Table \ref{table-capacity} summarises the $ \ISC $ values for different regimes and their respective implications for the two-hidden layered DBN. For example if $ \alpha_1 > \frac{1}{\gamma}$  then the capacity is dictated only by the number of hidden units in the second layer and increasing $ \alpha_1 $ has no effect. Multi-layering should be considered to handle distributions with multiple modes. Also, considering a practical scenario where there is a computational and memory constraint that translates into a budget on the number of parameters, i.e. $\alpha_1(1 + \alpha_2) = c $, we get the optimal distribution of hidden units in the two layers that maximizes the capacity. In particular if $ c < 1 $ then it is recommended to allocate all hidden units to layer 1 itself instead of adding more layers. 

\section{Experimental Results}
\label{section-experiments}
\noindent
Our main goals are to experimentally verify Theorems \ref{thm-main}, \ref{thm-dbn} and
Corollaries \ref{cor-fixed-budget} \ref{cor:single-double-comparison}.
\begin{comment}
Our main goals while conducting experiments were
\begin{enumerate}
\item
To verify the correctness of the main formulae given in Theorems \ref{thm-main} and \ref{thm-dbn} using computer simulations.
\item
To verify and validate the claim of Corollary \ref{cor-fixed-budget} on the standard MNIST dataset. If we keep the number of parameters constant then can we get a better log-likelihood by changing the architecture?
%\item
%To verify and validate on a synthetic dataset that if the number of modes in the input distribution is large, then an RBM with small number of hidden units shall not be able to represent it.
\item
To verify whether the given a single layer RBM architecture can we train a two layer DBM with far less number of parameters that gives the same log-likelihood to the training set? (Corollary~\ref{cor:single-double-comparison})
\end{enumerate}
\end{comment}
\noindent
All experiments were run on CPU with 2 Xeon Quad-Core processors (2.60GHz 12MB L2 Cache) and 16GB memory running Ubuntu 16.02 \footnote{The source code and instructions to run is available at \url{http://mllab.csa.iisc.ernet.in/publications}.}.
\subsection{Validating estimate of Number of modes}
\noindent
To verify our theoretical claims of Theorems \ref{thm-main} and \ref{thm-dbn} a number of simulation experiments for varied number of visible and hidden units were conducted. \textbf{To enable execution of exhaustive tests in reasonable time, the values of $ n $ had to be kept small}. The entries $ \{w_{ij}\} $ of the weight matrix were drawn from an i.i.d. mean zero normal distribution. Each of the $ 2^n - 1 $ vectors (leaving out the trivial all zero vector) was then tested for being \emph{perfectly reconstructible}. A comparison of the theoretical predictions and experimental results is shown in Figures \ref{fig-comparison} and \ref{fig-comparison-layers} for single layer and two layer RBMs respectively. It can be seen that the theoretical predictions follow similar trend as the experimental results.

\begin{figure}[!t]
\centering
{\includegraphics[width=3.0in]{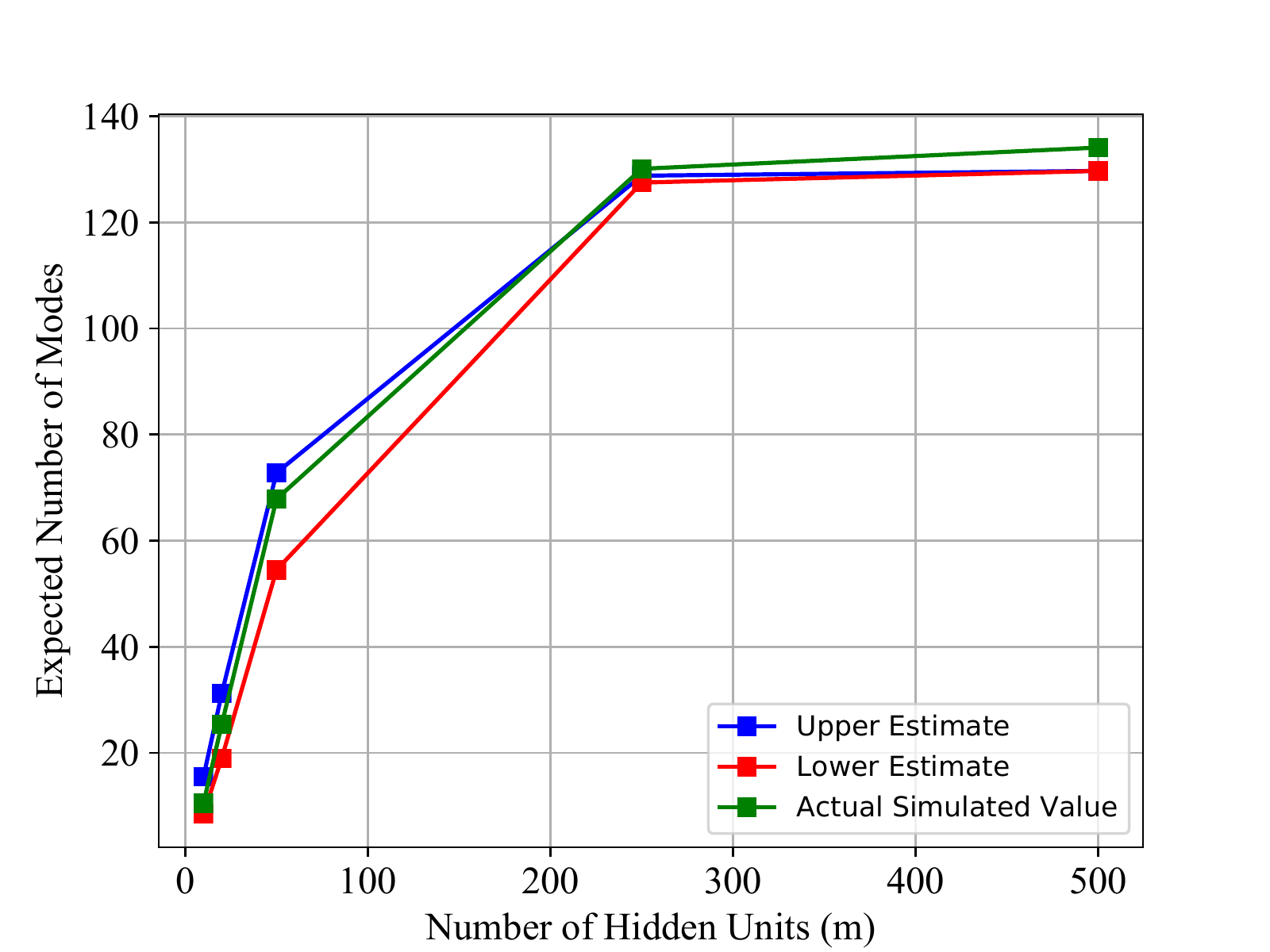}}%
%\hfil\
\caption{Comparison chart of the upper and lower estimates with the actual simulation value of expected number of modes ($ 2^{n\C(n,m)} $) for $ n = 12 $. }
\label{fig-comparison}
\end{figure}

\begin{table}
\caption{Actual $ \ISC $ for $ \textbf{\RBM}_{n,m} $ for $ m = 500 $, obtained by averaging brute-force enumeration from 2000 independent instantiations of weight matrix, i.e., $ \mathcal{C}(n,m) = \frac{1}{n}\log_2 \left(\frac{1}{2000}\sum_{i=1}^{2000}k_i \right) $ where $ k_i $ is the number of modes enumerated in $ i $th instantiation.}
\label{table-capacity-limit}
\begin{center}
	\begin{tabular}{ c c c c }
		\toprule
		& $ n = 10 $ & $ n = 11 $ & $ n = 12 $\\
		\midrule
	$ \mathcal{C}(n,m) $ & 0.585 & 0.585 & 0.588 \\
		\bottomrule
	\end{tabular}
\end{center}
\end{table}

\noindent
{\bf Discussion.} Figure \ref{fig-comparison} shows that the predicted bounds on the modes are close to the actual modes enumerated. Table \ref{table-capacity-limit} validates the claim that for an $ \textbf{\RBM}_{n,m} $ as $ m \to \infty $,  $ \ISC \to 0.585$ (Corollary \ref{cor:m-inf}). To enable bruteforce enumeration in reasonable time the values for $ n $ had to be kept small. %This at first sight seems to contradict the well-established result that by increasing $ m $ any distribution on input vectors can be modelled arbitrarily well \cite{le2008representational}. The reason for this seeming discrepancy is that here the $ \ISC $ is being measured in the expected sense. So although there would exist an optimal weight matrix that could model any input distribution, if the cardinality of the support set of input distribution is large, then finding such a weight matrix through training would not be practical. 
Figure \ref{fig-comparison-layers} in the supplementary section shows the theoretical upper bound and actual simulated $ \ISC $ values for a DBM with 2 hidden layers if we fix the total number of hidden units ($ m_1 + m_2 = 10 $)  and vary the ratio $ \beta = \frac{m_2}{m_1} $. %For the case when $ n = 10 $ the $ \ISC $ is maximum when $ \beta = 0 $. This is as predicted by Corollary \ref{cor-fixed-budget}. 
%As predicted by Corollary \ref{cor-fixed-budget}, $ \exists \beta^* $ for which $ \ISC $ is maximum. 
It can be seen that both theoretical prediction of $ \ISC $ and actual simulation results are closely aligned.

%\subsection{Relation with $ n $}
%To check how the expected capacity varies with $ n$ and $m $, the logarithm of the expected capacity was plotted against $ n$ and $ m $. It was seen that if $ n = m = \frac{N}{2} $ varies then the logarithm of the expected capacity varies linearly with a slope of $ 0.11 $ with $ N $. This implies $ \mathbb{E}\left[ \Omega_{n,n} \right] \sim \exp (0.11N) $. A 2D heat map is shown in Figure \ref{fig-visible-hidden}

%\begin{figure}[h]
%  \centering
%      \includegraphics[height=2in]{visible-hidden}
%      \caption{Heat map showing variation of the Logarithm of the expected capacity with $ n $ and $ m $. It can be seen that for a fixed value of $ n $($ m$), the expected capacity increases with $ m $($ n $) and tapers to a particular value.}
%      \label{fig-visible-hidden}
%\end{figure}

\subsection{DBM design under budget on parameters}
\noindent

To validate the claim made in Corollary \ref{cor-fixed-budget} we considered training a DBM with two hidden layers on the MNIST dataset. %The dataset consists of 28x28 images of handwritten digits split into 60000 training and 10000 test images. 
For this dataset, the standard architecture for a two hidden layer DBM uses $ m_1 = 500, m_2 = 1000 $ hidden units (784x500x1000) \cite{salakhutdinov2009deep,salakhutdinov2010efficient,hinton2012better}. In this case $ \alpha_1 = 0.64, \alpha_2 = 1.27$ and the number of parameters $ = 784 \times 500 + 500 \times 1000 + (784 + 500 + 1000) = 894284 $. Under a budget of fixed number of parameters Corollary \ref{cor-fixed-budget} suggests a better split of the number of hidden units. Accordingly we trained a DBM, with architecture of 784x945x161(\emph{Recommended}), with 894915 parameters. We note that the number of parameters are similar to the standard architecture of 784x500x1000 (\emph{Classical}), with 894284  parameters. 

%We thus trained two DBMs, one with the standard architecture of 784x500x1000 ($ A $), with 894284 parameters and the other with 784x695x500 ($ B $), with 894359 parameters. We note that the number of parameters for both architectures is similar.
%\end{itemize}

\noindent
We used the standard metric average log-likelihood of test data \cite{salakhutdinov2009deep,salakhutdinov2010efficient} as the measure to compare. To estimate the model’s partition function we used 20,000 $\beta_k$ spaced uniformly from 0 to 1.0. 

\noindent
{\bf Discussion. } The classical \emph{tuned} architecture for training a DBM with 2 hidden layers for the original MNIST dataset gives a log-likelihood of \emph{-84.62}. Using our recommended architecture, we were able to get a matched log-likelihood of \emph{-84.29} without significant tuning. 

\subsection{Wide single layer RBM vs \emph{lean} two-layered DBM}
To verify our claim in Corollary~\ref{cor:single-double-comparison} we chose single layer RBMs with $ n = 20 $ and $ n=30 $ and varying $ \alpha = \frac{m}{n} \in \{3,7,10,15\} $. We initialized weights and biases of each RBM architecture randomly and then performed gibbs sampling for 5000 steps to generate a synthetic dataset of 60,000 points. The same dataset was then used for training and evaluating corresponding multilayer DBM architecture suggested by our formula. The resulting test-set log likelihood are depicted in Figure~\ref{fig-dbm-compare}.

\begin{figure}[h]
  \centering
      \includegraphics[width=2.5in]{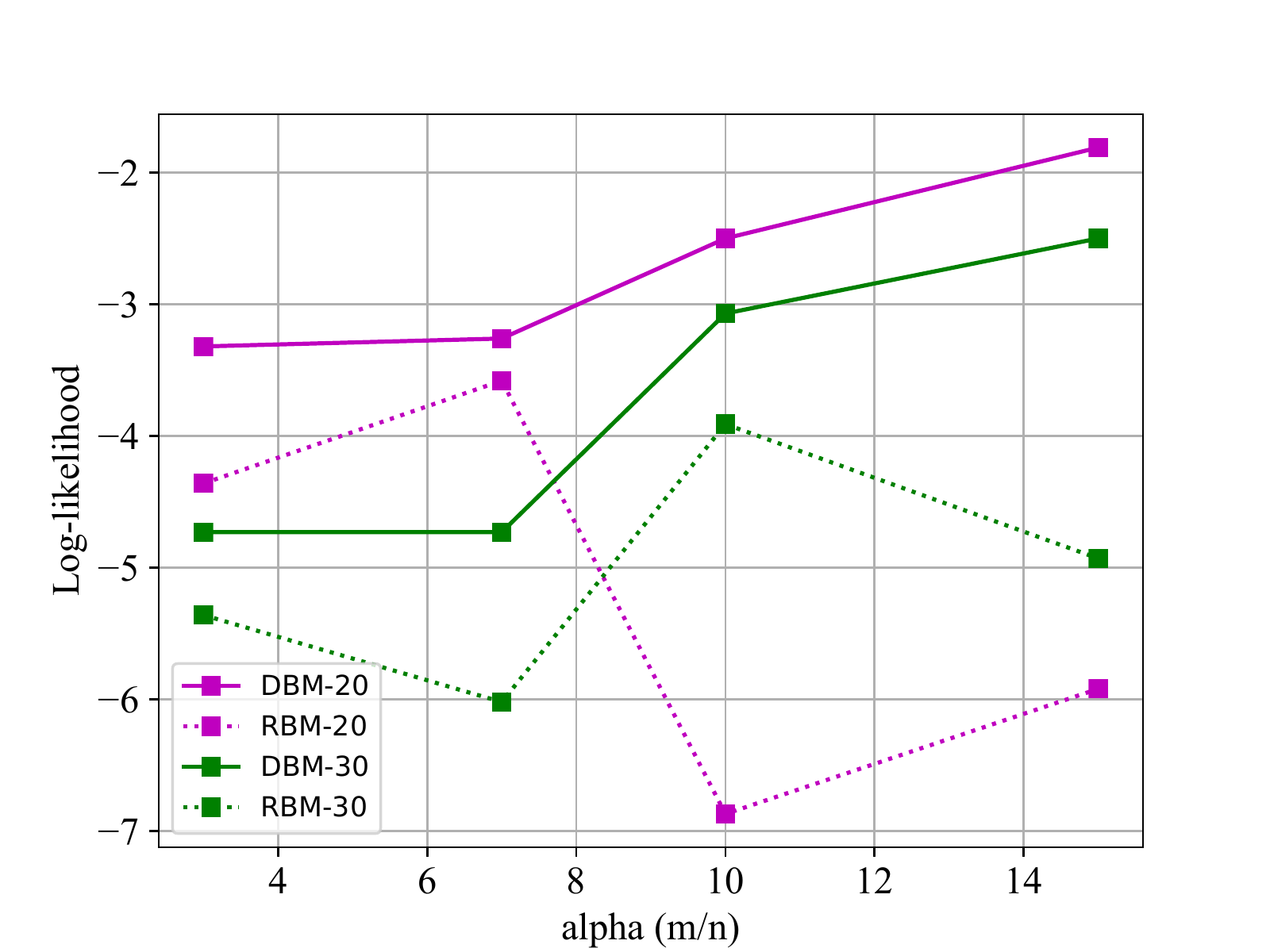}
      \caption{Comparison of test set log-likelihood attained for single layer RBM and two-layer DBM for $ n=20 $ and $ n=30 $. It can be seen that the DBM with much less parameters gives atleast as good log-likelihood as RBM.}
      \label{fig-dbm-compare}
\end{figure}

\begin{comment}
\begin{table}[h!]
	\centering
	\begin{tabular}{ c c c }
		& RBM & DBM\\
		\hline
		arch.& 20x60 & 20x30x10\\
		log-prob &  -4.36 & -3.32 \\
		\hline
		arch.& 20x100 & 20x35x15\\
		log-prob &  -2.27 & -3.02 \\
		\hline
		arch.& 20x140 & 20x35x15\\
		log-prob &  -3.58 & -3.258 \\
		\hline
		arch.& 20x200 & 20x35x15\\
		log-prob &  -6.87 & -2.50 \\
		\hline
	\end{tabular}
	\caption{Visible Units : 20}
\end{table}

\begin{table}[h!]
	\centering
	\begin{tabular}{ c c c }
		& RBM & DBM\\
		\hline
		arch.& 30x90 & 30x43x13\\
		log-prob &  -5.36 & -4.73 \\
		\hline
		arch.& 30x150 & 30x45x15\\
		log-prob &  -3.09 & -3.10 \\
		\hline
		arch.& 30x210 & 30x47x17\\
		log-prob &  -6.02 & -4.725 \\
		\hline
		arch.& 30x300 & 30x47x17\\
		log-prob &  -3.91 & -3.07 \\
		\hline
	\end{tabular}
	\caption{Visible Units : 30}
\end{table}
\end{comment}
\noindent
\textbf{Discussion.} We can see that optimal DBM architecture gives same or improved log-likelihood despite the fact that it has less number of parameters than the respective single layer RBM, thus justifying our claim. 

\section{Conclusion}
We studied the IS formalism, first introduced to study Spin glasses, to understand the energy landscape of one and two layer DBMs and proposed $ \ISC $, a measure of representation power of RBMs.   
 \ISC~ makes practical suggestions such as whenever  number of hidden units $ m > 20n$, the  \ISC~  saturates and multilayering should be considered. Also, \ISC~ suggests alternative two layer architectures  to single layer RBMs  which have equal or more representational power with far fewer number of parameters.

% conference papers do not normally have an appendix

% use section* for acknowledgment
\section*{Acknowledgment}
The authors would like to thank the referees for their insightful comments.
CB gratefully acknowledge partial support from a generous grant from Microsoft Research India.

%\end{multicols}
\bibliography{rbm3}

\begin{thebibliography}{10}

\bibitem{auffinger2013random}
Antonio Auffinger, G{\'e}rard~Ben Arous, and Ji{\v{r}}{\'\i} {\v{C}}ern{\`y}.
\newblock Random matrices and complexity of spin glasses.
\newblock {\em Communications on Pure and Applied Mathematics}, 66(2):165--201,
  2013.

\bibitem{barr1999mean}
Donald~R Barr and E~Todd Sherrill.
\newblock Mean and variance of truncated normal distributions.
\newblock {\em The American Statistician}, 53(4):357--361, 1999.

\bibitem{biroli2000inherent}
Giulio Biroli and R{\'e}mi Monasson.
\newblock From inherent structures to pure states: Some simple remarks and
  examples.
\newblock {\em EPL (Europhysics Letters)}, 50(2):155, 2000.

\bibitem{bray1980metastable}
AJ~Bray and MA~Moore.
\newblock Metastable states in spin glasses.
\newblock {\em Journal of Physics C: Solid State Physics}, 13(19):L469, 1980.

\bibitem{cueto2010geometry}
Mar{\'\i}a~Ang{\'e}lica Cueto, Jason Morton, and Bernd Sturmfels.
\newblock Geometry of the restricted boltzmann machine.
\newblock {\em Algebraic Methods in Statistics and Probability,(eds. M. Viana
  and H. Wynn), AMS, Contemporary Mathematics}, 516:135--153, 2010.

\bibitem{de1980white}
C~De~Dominicis, M~Gabay, T~Garel, and Henri Orland.
\newblock White and weighted averages over solutions of thouless anderson
  palmer equations for the sherrington kirkpatrick spin glass.
\newblock {\em Journal de Physique}, 41(9):923--930, 1980.

\bibitem{freund1992unsupervised}
Yoav Freund and David Haussler.
\newblock Unsupervised learning of distributions on binary vectors using two
  layer networks.
\newblock In {\em Advances in neural information processing systems}, pages
  912--919, 1992.

\bibitem{hinton2012better}
Geoffrey~E Hinton and Ruslan~R Salakhutdinov.
\newblock A better way to pretrain deep boltzmann machines.
\newblock In {\em Advances in Neural Information Processing Systems}, pages
  2447--2455, 2012.

\bibitem{le2008representational}
Nicolas Le~Roux and Yoshua Bengio.
\newblock Representational power of restricted boltzmann machines and deep
  belief networks.
\newblock {\em Neural computation}, 20(6):1631--1649, 2008.

\bibitem{martens2013representational}
James Martens, Arkadev Chattopadhya, Toni Pitassi, and Richard Zemel.
\newblock On the representational efficiency of restricted boltzmann machines.
\newblock In {\em Advances in Neural Information Processing Systems}, pages
  2877--2885, 2013.

\bibitem{montufar2014deep}
Guido Mont{\'u}far.
\newblock Deep narrow boltzmann machines are universal approximators.
\newblock {\em arXiv preprint arXiv:1411.3784}, 2014.

\bibitem{montufar2011refinements}
Guido Montufar and Nihat Ay.
\newblock Refinements of universal approximation results for deep belief
  networks and restricted boltzmann machines.
\newblock {\em Neural Computation}, 23(5):1306--1319, 2011.

\bibitem{montufar2017hierarchical}
Guido Mont{\'u}far and Johannes Rauh.
\newblock Hierarchical models as marginals of hierarchical models.
\newblock {\em International Journal of Approximate Reasoning}, 88:531--546,
  2017.

\bibitem{montufar2015does}
Guido~F Mont{\'u}far and Jason Morton.
\newblock When does a mixture of products contain a product of mixtures?
\newblock {\em SIAM Journal on Discrete Mathematics}, 29(1):321--347, 2015.

\bibitem{montufar2011expressive}
Guido~F Mont{\'u}far, Johannes Rauh, and Nihat Ay.
\newblock Expressive power and approximation errors of restricted boltzmann
  machines.
\newblock In {\em Advances in neural information processing systems}, pages
  415--423, 2011.

\bibitem{parisi1995mean}
Giorgio Parisi and Marc Potters.
\newblock Mean-field equations for spin models with orthogonal interaction
  matrices.
\newblock {\em Journal of Physics A: Mathematical and General}, 28(18):5267,
  1995.

\bibitem{roberts1981metastable}
SA~Roberts.
\newblock Metastable states and innocent replica theory in an ising spin glass.
\newblock {\em Journal of Physics C: Solid State Physics}, 14(21):3015, 1981.

\bibitem{salakhutdinov2009deep}
Ruslan Salakhutdinov and Geoffrey Hinton.
\newblock Deep boltzmann machines.
\newblock In {\em Artificial Intelligence and Statistics}, pages 448--455,
  2009.

\bibitem{salakhutdinov2010efficient}
Ruslan Salakhutdinov and Hugo Larochelle.
\newblock Efficient learning of deep boltzmann machines.
\newblock In {\em AISTATs}, volume~9, pages 693--700, 2010.

\bibitem{stein2013spin}
Daniel~L Stein and Charles~M Newman.
\newblock {\em Spin glasses and complexity}.
\newblock Princeton University Press, 2013.

\bibitem{stillinger1982hidden}
Frank~H Stillinger and Thomas~A Weber.
\newblock Hidden structure in liquids.
\newblock {\em Physical Review A}, 25(2):978, 1982.

\bibitem{tanaka1980analytic}
F~Tanaka and SF~Edwards.
\newblock Analytic theory of the ground state properties of a spin glass. i.
  ising spin glass.
\newblock {\em Journal of Physics F: Metal Physics}, 10(12):2769, 1980.

\bibitem{van2011discriminative}
Laurens van~der Maaten.
\newblock Discriminative restricted boltzmann machines are universal
  approximators for discrete data.
\newblock Technical report, Technical Report EWI-PRB TR 2011001, Delft
  University of Technology, 2011.

\end{thebibliography}
\bibliographystyle{ieeetr/unsrtnat}
\pagebreak
\onecolumn
\appendix
\section*{Supplementary material}
In the following sections we provide additional material (proofs and figures) that supplement our main results. Section \ref{section:facts} outlines the preliminary facts and notations that we use for the proofs. The subsequent sections provide the detailed proofs for respective lemmas and theorems. Figure \ref{fig-comparison-layers} compares the theoretical upper bound estimate with the actual simulated values for modes of two layer DBMs ($ \C(n,m_1,m_2) $).

\section{Preliminary Facts and Notations}
\label{section:facts}
In the proofs that follow we use the following facts and notations:
\begin{enumerate}
\item
The probability density function (pdf) of standard normal distribution $ \mathcal{N}(0,1) $
\[ \phi(x) = \frac{1}{\sqrt{2\pi}}\exp\left(-\frac{x^2}{2}\right) \]
\item
The cumulative distribution function (cdf) of standard normal distribution
\[ \Phi(x) =  \int_{-\infty}^{x}\phi(x)dx = \frac{1}{2}\left[ 1 + \erf\left( \frac{x}{\sqrt{2}} \right) \right]  
\text{ where } \erf(x) = \frac{1}{\sqrt{\pi}}\int_{-x}^{x} e^{-t^2}dt \]
\item
The pdf of a skew normal distribution $ \hat{\mathcal{N}} $ with skew parameter $ \alpha $
\[ f(x) = 2\phi(x)\Phi(\alpha x) \]
\item
If $ X \sim \mathcal{N}(\mu,\sigma^2) $, $ a \in \mathds{R}, \alpha = \frac{a - \mu}{\sigma} $, then $ X $ conditioned on $ X > a $ follows a truncated normal distribution with moments
\begin{eqnarray*}
\mathds{E}\left[ X | X > a \right] &=& \mu + \sigma\frac{\phi(\alpha)}{Z} \\
Var(X | X > a) &=& \sigma^2\left[ 1 + \alpha\frac{\phi(\alpha)}{Z} - \left( \frac{\phi(\alpha)}{Z} \right)^2 \right] 
\end{eqnarray*}
where $ Z = 1 - \Phi(\alpha) $.
\item
\emph{Squeeze Theorem}\footnote{\url{http://mathonline.wikidot.com/the-squeeze-theorem-for-convergent-sequences}}: Let, $ \{a_m\},\{b_m\}, \{c_m\} $ be sequences such that $ \forall  m \ge m_0 $ ($ m_0 \in \mathds{R} $)
\[ a_m \le b_m \le c_m \]
Further, let $ \lim_{m\to \infty} a_m = \lim_{m\to \infty} c_m = L $, then
\[ \lim_{m\to \infty} b_m = L \]
\end{enumerate}

\section{Proof of Lemma~\ref{lem:equiv-prv-ofs} (See page~\pageref{lem:equiv-prv-ofs})}
\lemequiv*
\begin{proof}
Let $ \textbf{h}^* = \up(\textbf{v}) $ (conditioning on $ \theta $ is implicit). If $ \textbf{v} $ is perfectly reconstructible $\implies  \textbf{v} = \argmax_{\textbf{v}} P(\textbf{v}|\textbf{h}^*) \implies \forall \textbf{v}' \neq \textbf{v},  P(\textbf{v}',\textbf{h}^*) < P(\textbf{v},\textbf{h}^*)  $. Similarly since $\textbf{h}^* = \argmax_{\textbf{h}} P(\textbf{h}|\textbf{v}), \forall \textbf{h}' \neq \textbf{h}^*,  P(\textbf{v},\textbf{h}') < P(\textbf{v},\textbf{h}^*)$. Hence the state $ \{\textbf{v},\textbf{h}^*\} $ is stable against any number of flips of visible units and against any number of flips of hidden units, $\implies \{\textbf{v},\textbf{h}^*\}$ is atleast one-flip stable.\\
Conversely let $\{\textbf{v}^*,\textbf{h}^*\}$ be one-flip stable. We shall prove  by contradiction that $ \up(\textbf{v}^*) = \textbf{h}^* $ and $ \down(\textbf{h}^*) = \textbf{v}^* $. Assume $ \up(\textbf{v}^*) = \textbf{h}' \neq \textbf{h}^*$. We use the fact that for an RBM the hidden units are conditionally independent of each other given the visible units. Thus $ \textbf{h}' = \argmax_{\textbf{h}} P(\textbf{h}|\textbf{v}^*) = \{ \argmax_{h_j} P(h_j|\textbf{v}^*) \}_{j=1}^m$. Further $ P(\textbf{h}^*|\textbf{v}^*) = \prod_{j=1}^m P(h_j^*|\textbf{v}^*) $. Let $ k $ be an index such that $ h_k' \neq h_k^* $. Since $ h_k' = \argmax_{h_k} P(h_k|\textbf{v}^*), \implies P(h_k'|\textbf{v}^*) > P(h_k^*|\textbf{v}^*)$. Moreover, $ P(\textbf{v}^*,\textbf{h}^*) = P(\textbf{v}^*)P(\textbf{h}^*|\textbf{v}^*) = P(\textbf{v}^*)\prod_{j=1}^m P(h_j^*|\textbf{v}^*) $. Thus just by flipping $ h_k^* $ to $ h_k' $ we can increase the probability of the state $ \{\textbf{v}^*, \textbf{h}^*\} $. This contradicts the one-flip stability hypothesis. Similarly using the conditional independence of visible units given the hidden units we can show that $\down(\textbf{h}^*) = \textbf{v}^* $.
\end{proof}

\section{Proof of Lemma~\ref{lem-indep} (See page~\pageref{lem-indep})}
\lemindep*
\begin{proof}
We first note that given a visible vector $ \textbf{v}  \in \{0,1\}^n$ the most likely configuration of the hidden vector 
\begin{center}
$ \left \lbrace h_j = \left[\up(\textbf{v})\right]_j = \O1_{\left[\sum_{i = 1}^n  w_{ij}v_i > 0 \right]} \right \rbrace_{j=1}^m$ 
\end{center}
Likewise given a hidden vector \textbf{h}, the most likely visible vector 
\begin{center}
$\left \lbrace v_i = \left[\down(\textbf{h})\right]_i = \O1_{\left[\sum_{j = 1}^m  w_{ij}h_j > 0 \right]}  \right \rbrace_{i=1}^n$
\end{center}
\textbf{Case 1: $ r = 1 $}\\
By symmetry it can be assumed $ v_1 = 1 $, and $ v_i = 0 (\forall i > 1)$. Then $ \left \lbrace h_j =  \O1_{\left[  w_{1j} > 0 \right]} \right \rbrace_{j=1}^m$ . Since each of $ w_{1j} $ is i.i.d. as per $ \mathcal{N}(0,\sigma^2) $, $ h_j $ is a Bernoulli random variable with $ P(h_j = 1) = \frac{1}{2} $. Again by symmetry it is assumed the first $ l $ units $ \{h_j\}_{j=1}^l $ are one. Then the most likely \emph{reconstructed} visible vector is given by $ \left \lbrace \hat{v}_i =  \O1_{\left[X_i =\sum_{j = 1}^{l}  w_{ij} > 0 \right]} \right \rbrace_{i=1}^n$. Since $w_{1j} > 0$ for all $ 1 \le j \le l \implies \hat{v}_1 = 1$.
Also, for all $ i > 1, w_{ij} \sim \mathcal{N}(0, \sigma^2) \implies X_i\sim \mathcal{N}(0, l\sigma^2) \implies \{\hat{v}_i\}_{i >1} $ is a Bernoulli random variable with $ \left \lbrace P[\hat{v}_{i} = 1] = \frac{1}{2}\right \rbrace_{i = 2}^n $. The result then follows by mutual independence of  $ \hat{v}_i $.

\textbf{Case 2: $ r > 1 $}\\
\noindent
For $ r (> 1) $ ones in $ \textbf{v} $ and $ l $ ones in $ \textbf{h} = \up(\textbf{v}) $ the problem of computing $ \left \lbrace P[\hat{v}_i  = 1]\right \rbrace_{i=1}^r $ can be reformulated in terms of matrix row and column sums, viz, given $ W \in \mathds{R}^{r\times l} $ where all entries $ w_{ij} \sim \mathcal{N}(0,\sigma^2)$ are i.i.d. and given that all the column sums $ \left \lbrace C_j = \sum_{i = 1}^r w_{ij} > 0 \right \rbrace_{j=1}^l $, to compute the probability that all the row sums are positive, i.e., $ \left \lbrace R_i = \sum_{j = 1}^l w_{ij} > 0 \right \rbrace_{i=1}^r $.

Using properties of normal distribution it can be shown that conditioned on the fact that $ C_j > 0 $, the posterior distribution of $ w_{ij} $ shall be \emph{skew-normal} with mean $ \mu_{ij} = \sigma \sqrt{\frac{2}{\pi r}} $ and variance $ \sigma_{ij}^2 = \sigma^2 \left(1 -\frac{2}{\pi r}\right) $. Since the random variables $ \left \lbrace w_{ij} |  C_j > 0 \right \rbrace_{j=1}^l$ are independent the posterior mean of $ R_i $ shall be $ \tilde{\mu}_{i} = l\sigma \sqrt{\frac{2}{\pi r}} $ and the posterior variance $ \tilde{\sigma}_{i}^2 = l\sigma^2 \left(1 -\frac{2}{\pi r}\right)$. Since $ l \gg 1 $ by \textit{Central Limit Theorem} $ R_i $ follow a normal distribution. Since the $ R_i $ are negatively correlated (proof follows) and $ \left \lbrace P[\hat{v}_i  = 1] = \frac{1}{2}\right  \rbrace_{i > r} $ by similar reasoning as in Case 1 we get our desired upper bound.

{\bf Negatively Correlated $ R_i $'s:}
Conditioned on the fact $ \left \lbrace C_j > 0 \right \rbrace_{j=1}^l $ the random variables $ \left \lbrace R_i \right \rbrace_{i=1}^r $ are not independent. They are negatively correlated because for all $ R_i, R_t (t \neq i)$, 
\[ P(R_i > 0 | \left \lbrace C_j > 0 \right \rbrace_{j=1}^l, R_t > 0 ) < P(R_i > 0 | \left \lbrace C_j > 0 \right \rbrace_{j=1}^l) \]
Hence the expression given in Lemma \ref{lem-indep} is an upper bound since we have neglected the negative correlation among the $ R_i $ and in the process over-estimated the probabilities.

\end{proof}

\section{Proof of Lemma~\ref{lem-tild-constants} (See page~\pageref{lem-tild-constants})}
\lemtildconst*
\begin{proof}
The conditional distribution for $ R_1 = \sum_{j = 1}^l w_{1j}$ is obtained from the proof of Lemma \ref{lem-indep}. 
\begin{center}
$ \left(R_1 | \{C_j > 0\}_{j=1}^l\right) \sim \mathcal{N}\left(\tilde{\mu}_{1}, \tilde{\sigma}_{1}^2 \right)$
\end{center}
where $ \tilde{\mu}_{1} = l\sigma \sqrt{\frac{2}{\pi r}}$, $\tilde{\sigma}_{1}^2 = l\sigma^2 \left(1 -\frac{2}{\pi r}\right) $.
Using similar arguments as in proof of Lemma \ref{lem-indep}, conditioned on $ R_t > 0 $ the posterior distribution of $ w_{tj}$ shall be skew normal $  \hat{\mathcal{N}}\left[\sigma \sqrt{\frac{2}{\pi l}}, \sigma^2 \left(1 -\frac{2}{\pi l}\right)\right] $.
Then conditioned on $ \{R_t > 0\}_{t=1}^{i-1} $, $C_j$ shall be distributed as per skew normal
\[\left(C_j | \{R_t > 0\}_{t=1}^{i-1}\right) \sim \hat{\mathcal{N}}(\mu_c, \sigma_c^2)\]  where\[\mu_c = (i-1)\sigma \sqrt{\frac{2}{\pi l}} \text{ and } \sigma_c^2 = (i-1)\sigma^2 \left(1 -\frac{2}{\pi l}\right) + (r-i+1)\sigma^2\]
\noindent
Here we approximate the above distribution to be Normal since if $ i $ is large then \emph{Central Limit Theorem} would be applicable, otherwise the normally distributed variables $ \{w_{kj}\}_{k=i}^r $ would dominate the sum. Then conditioned on $ \{R_t > 0\}_{t=1}^{i-1}, C_j > 0$, $ C_j $ shall be distributed as per truncated normal distribution  \cite{barr1999mean}  with moments  
\begin{eqnarray*}
\mathds{E}\left[C_j | \{R_t > 0\}_{t=1}^{i-1}, C_j > 0 \right]&= \tilde{\mu}_c =& \mu_c + \sigma_c \frac{\phi}{Z} \\
\text{Var}\left[C_j | \{R_t > 0\}_{t=1}^{i-1}, C_j > 0 \right] &= \tilde{\sigma}_c^2 =& \sigma_c^2 \left[ 1 - \frac{\mu_c \phi}{\sigma_c Z} - \frac{\phi^2}{Z^2} \right]
\end{eqnarray*}
where $ \sigma_c^2 = (i-1)\sigma^2 \left(1 -\frac{2}{\pi l}\right) + (r-i+1)\sigma^2,\\ \mu_c = (i-1)\sigma \sqrt{\frac{2}{\pi l}}, Z = \frac{1}{2} - \frac{1}{2}\erf\left(-\frac{\mu_c}{\sigma_c\sqrt{2}}\right) $ and
$ \phi = \frac{1}{\sqrt{2\pi}}e^{\left(-\frac{\mu_c^2}{2\sigma_c^2}\right)} $. Then
$ \mathds{E}\left[w_{ij} | \{R_t > 0\}_{t=1}^{i-1}, C_j = c\right] = (c - \mu_c)\frac{\sigma^2}{\sigma_c^2} $ and $ \text{Var}\left[w_{ij} | \{R_t > 0\}_{t=1}^{i-1}, C_j = c \right] = \sigma^2 \left(1 - \frac{\sigma^2}{\sigma_c^2}\right) $.
The result then follows from Laws of total expectation and total variance respectively.
\end{proof}
\begin{remark}
The random variables $ \{\tilde{w}_{ij}= w_{ij} | \{R_t > 0\}_{t=1}^{i-1}, C_j > 0\}_{j=1}^l $ shall be negatively correlated with one another so we should subtract the covariance terms while determining the effective variance of $ R_i = \sum_{j=1}^l \tilde{w}_{ij} $. Thus if we don't subtract the covariance terms from the variance we would get a lower bound on the posterior probability of $ R_i$ being positive. However it is close as can be seen in Figure \ref{fig-comparison}.
\end{remark}

\section{Proof of Theorem~\ref{thm-main} (See page~\pageref{thm-main})}
\thmcapacity*
\begin{proof}
The upper bound follows from Lemma \ref{lem-indep} and applying linearity of expectation. 

\[U_{n,m} = \sum_{r = 1}^n \binom{n}{r} \sum_{l=1}^m \binom{m}{l} \left(\frac{1}{2}\right)^m \left[ \frac{1}{2} - \frac{1}{2}\erf \left( - \sqrt{\frac{l}{\pi r - 2}}\right) \right]^r \left(\frac{1}{2}\right)^{n-r}\]

For lower bound, we use Lemma \ref{lem-tild-constants}. We have $ \mathds{E}\left[w_{ij} | \{R_t > 0\}_{t=1}^{i-1}, C_j > 0 \right] = \tilde{\mu}_{i}(r,l) $ and $ \text{Var}\left[w_{ij} | \{R_t > 0\}_{t=1}^{i-1}, C_j > 0 \right] = (\tilde{\sigma}_{i}(r,l))^2 $. Thus posterior mean and variance of $ \{R_i\}_{i=1}^r $ shall be $ l\tilde{\mu}_{i}(r,l) $ and $ l(\tilde{\sigma}_{i}(r,l))^2 $ respectively. Then summing over all possibilities of $ l $ and applying linearity of expectation we get the lower bound.

\[ L_{n,m} = \sum_{r=1}^n \binom{n}{r}\sum_{l=1}^m \binom{m}{l} \left(\frac{1}{2}\right)^m \left \lbrace\prod_{i = 1}^{r}\left[ \frac{1}{2} - \frac{1}{2}\erf \left(-\frac{\tilde{\mu}_{i}(r,l)\sqrt{\frac{l}{2}}}{\tilde{\sigma}_{i}(r,l)}\right) \right]\right \rbrace \left(\frac{1}{2}\right)^{n-r}  \]

\end{proof}

\section{Proof of Corollory \ref{cor:m-inf} (See page~\pageref{cor:m-inf})}
\corlimit*
%\begin{corollary}
%For the set $ \textbf{\RBM}_{n,m} $ in the limit $ m \to \infty $ following equalities hold \textbf{almost surely}.
%\[ \lim_{m \to \infty} L_{n,m} = \lim_{m \to \infty} \mathcal{C}(n,m)  = \lim_{m \to \infty} U_{n,m} = (1.5)^n \]
%\end{corollary}
\begin{proof} 
We shall show that $ \lim_{m \to \infty} U_{n,m} \le 1.5^n $ and $ \lim_{m \to \infty} L_{n,m} \ge 1.5^n $. Then using \emph{Squeeze Theorem} and the fact that limits preserve inequalities the result shall hold.
\[\lim_{m \to \infty} U_{n,m} = \lim_{m \to \infty} \left\lbrace\sum_{r = 1}^n \binom{n}{r} \sum_{l=1}^m \binom{m}{l} \left(\frac{1}{2}\right)^m \left[ \frac{1}{2} - \frac{1}{2}\erf \left( - \sqrt{\frac{l}{\pi r - 2}}\right) \right]^r \left(\frac{1}{2}\right)^{n-r}\right \rbrace\]
If we replace the $ l $ inside the $ \erf $ function by $ m $ then we would be increasing the value of the expression since $ m \ge l $. Thus
\begin{eqnarray*}
\lim_{m \to \infty} U_{n,m} &\le& \lim_{m \to \infty} \left\lbrace\sum_{r = 1}^n \binom{n}{r} \sum_{l=1}^m \binom{m}{l} \left(\frac{1}{2}\right)^m \left(\frac{1}{2}\right)^r \left[ 1 - \erf \left( - \sqrt{\frac{m}{\pi r - 2}}\right) \right]^r \left(\frac{1}{2}\right)^{n-r}\right \rbrace \\
&=& \lim_{m \to \infty} \sum_{r = 1}^n \binom{n}{r} \sum_{l=1}^m \binom{m}{l} \left(\frac{1}{2}\right)^m \left(\frac{1}{2}\right)^r \left[ 2 \right]^r \left(\frac{1}{2}\right)^{n-r} \\
&=& 1.5^n
\end{eqnarray*}
To get a lower bound on $ L_{n,m} $ we choose a small fixed constant $ \epsilon > 0 $. Then
\begin{eqnarray*}
\lim_{m \to \infty} L_{n,m} &=& \lim_{m \to \infty} \sum_{r=1}^n \binom{n}{r}\sum_{l=1}^m \binom{m}{l} \left(\frac{1}{2}\right)^m \left \lbrace\prod_{i = 1}^{r}\left[ \frac{1}{2} - \frac{1}{2}\erf \left(-\frac{\tilde{\mu}_{i}(r,l)\sqrt{\frac{l}{2}}}{\tilde{\sigma}_{i}(r,l)}\right) \right]\right \rbrace \left(\frac{1}{2}\right)^{n-r}\\
&\ge& \lim_{m \to \infty} \sum_{r=1}^n \binom{n}{r}\sum_{l=m \epsilon}^m \binom{m}{l} \left(\frac{1}{2}\right)^m \left \lbrace\prod_{i = 1}^{r}\left[ \frac{1}{2} - \frac{1}{2}\erf \left(-\frac{\tilde{\mu}_{i}(r,l)\sqrt{\frac{l}{2}}}{\tilde{\sigma}_{i}(r,l)}\right) \right]\right \rbrace \left(\frac{1}{2}\right)^{n-r} \\
&\ge& \lim_{m \to \infty} \sum_{r=1}^n \binom{n}{r}\sum_{l=m \epsilon}^m \binom{m}{l} \left(\frac{1}{2}\right)^m \left \lbrace\prod_{i = 1}^{r}\left[ \frac{1}{2} - \frac{1}{2}\erf \left(-\frac{\tilde{\mu}_{i}(r,l)\sqrt{\frac{m\epsilon}{2}}}{\tilde{\sigma}_{i}(r,l)}\right) \right]\right \rbrace \left(\frac{1}{2}\right)^{n-r}
\end{eqnarray*}
Since $ \tilde{\mu}_{i}(r,l) $ and $ \tilde{\sigma}_{i}(r,l) $ are non-zero finite quantities regardless of the value of $ l $ amd $ m $ and $ \epsilon $ is a fixed non-zero constant,
\begin{eqnarray*}
\lim_{m \to \infty} L_{n,m} &\ge& \lim_{m \to \infty} \sum_{r=1}^n \binom{n}{r}\sum_{l=m \epsilon}^m \binom{m}{l} \left(\frac{1}{2}\right)^m \left \lbrace\prod_{i = 1}^{r}\left[ \frac{1}{2} - \frac{1}{2}\erf \underbrace{ \left(-\frac{\tilde{\mu}_{i}(r,l)\sqrt{\frac{m\epsilon}{2}}}{\tilde{\sigma}_{i}(r,l)}\right)}_{\to -\infty} \right]\right \rbrace \left(\frac{1}{2}\right)^{n-r} \\
&=& \lim_{m \to \infty} \sum_{r=1}^n \binom{n}{r}\underbrace{ \sum_{l=m \epsilon}^m \binom{m}{l} \left(\frac{1}{2}\right)^m}_{\text{Prob}(l > m\epsilon)} \left \lbrace\left(\frac{1}{2}\right)^r\left[ 2 \right]^r\right \rbrace \left(\frac{1}{2}\right)^{n-r}
\end{eqnarray*}
Since $ \epsilon $ is an arbitrarily small number that we have chosen and $ l $ denotes the number of successes in $ m $ Bernoulli trials,  Prob$( l >m\epsilon ) = 1$.
\[\implies \lim_{m \to \infty} L_{n,m} \ge 1.5^n  \]
\[\implies 1.5^n \le \lim_{m \to \infty} L_{n,m} \le \lim_{m \to \infty} \mathcal{C}(n,m)  \le \lim_{m \to \infty} U_{n,m} \le 1.5^n  \]
\end{proof}

\section{Proof of Theorem~\ref{thm-dbn} (See page~\pageref{thm-dbn})}
\thmdbmcap*
\begin{proof}
As shown in Figure \ref{fig:dbm-rbm} we construct a single layer $ \RBM_{n+m_2,m_1} $ that has the same bipartite connections as $ \textbf{\RBM}_{n,m_1,m_2} $. The expected number of perfectly reconstructible vectors for the single layer RBM can then be obtained from Equation \ref{eq-approx}.
\begin{eqnarray*}
\C(n + m_2,m_1) &\le& \frac{1}{n} \log_2 U_{n + m_2,m_1} = \frac{1}{n}\log_2 S\\ 
                &=& \frac{1}{n}\log_2\left[1 - \frac{1}{2}\erf \left( - \sqrt{\frac{u}{\pi l - 4}}\right) \right]^{l}
\end{eqnarray*}
However this quantity is an overestimate. This counts the number of pairs of vectors $ \{\textbf{v}, \textbf{h}_2\} $ such that $ \begin{pmatrix}
\textbf{v}\\ 
\textbf{h}_2
\end{pmatrix} $ is perfectly reconstructible for $ \RBM_{n+m_2,m_1} $. Among these, there can be vectors like $ \begin{pmatrix}
\textbf{v}^{(1)}\\ 
\textbf{h}_2^{(1)}
\end{pmatrix}  $ and $ \begin{pmatrix}
\textbf{v}^{(2)}\\ 
\textbf{h}_2^{(2)}
\end{pmatrix}  $ where $ \textbf{v}^{(1)} = \textbf{v}^{(2)} $ resulting in repetitions. Assuming such vectors $ \textbf{v}^{(i)} $ are uniformly distributed among the $ 2^n $ possibilities, we approximate the problem to the following. Given $ 2^n $ distinct vectors, we make $ S $ draws from them  uniformly randomly with replacement. The expected number of distinct vectors that result is given by $ 2^n\left[1 - \left(1 - \frac{1}{2^n}\right)^{S}\right] $. If $ S < \gamma2^n $ then binomial approximation an be applied and we get the desired result.
\end{proof}

\section{Proof of Corollary~\ref{cor-high-alpha1-med-alpha2} (See page~\pageref{cor-high-alpha1-med-alpha2})}
\corhonemtwo*
\begin{proof}
For $ \alpha_1 > \frac{1}{\gamma}$, $S= \left[1 - \frac{1}{2}\erf \left( - \sqrt{\frac{n\alpha_1}{\pi n(1+\alpha_2) - 4}}\right) \right]^{n(1+\alpha_2)} = 1.5^{n(1+\alpha_2)} $. 

Moreover for $ \alpha_2 < \gamma $, since $ S = 1.5^{n(1+\alpha_2)} < 1.5^{n(1+\gamma)} = 2^{n(1+\gamma)\log_2(1.5)} = 2^{0.614n} $($ < \gamma 2^n $ for reasonable choices of $ n $), we can apply binomial approximation and the result follows.
\end{proof}

\section{Proof of Corollary~\ref{cor-fixed-budget} (See page~\pageref{cor-fixed-budget})}
\corfb*
\begin{proof}
%\begin{comment}
We consider two regimes.

\textbf{Regime 1 ($ \alpha_1 \le 1 + \alpha_2 $)}

In this regime using Theorem~\ref{thm-dbn}, $ \mathcal{C}(n,m_1,m_2) \le \frac{1}{n}\log_2 S $ where
\[S = \left[1 - \frac{1}{2}\erf\left(-\sqrt{\frac{u}{\pi l - \underbrace{4}_{=\smallO(1)}}}   \right)\right]^l = \left[1 - \frac{1}{2}\erf\left(-\sqrt{\frac{\frac{nc}{\alpha_1}}{\pi n\alpha_1}}   \right)\right]^{n\alpha_1} = \left[1 - \frac{1}{2}\erf\left(-\sqrt{\frac{c}{\pi \alpha_1^2}}   \right)\right]^{n\alpha_1} \]
We will prove that $ \frac{\partial S}{\partial \alpha_1} > 0 $. Taking natural logarithm on both sides,
\[\ln S = n\alpha_1 \ln \left[1 - \frac{1}{2}\erf\left(-\sqrt{\frac{c}{\pi \alpha_1^2}}   \right)\right] \] 
\begin{eqnarray*}
\frac{1}{S}\frac{\partial S}{\partial \alpha_1} &=& n\ln \left[1 - \frac{1}{2}\erf\left(-\sqrt{\frac{c}{\pi \alpha_1^2}}   \right)\right] + \frac{n\alpha_1}{1 - \frac{1}{2}\erf\left(-\sqrt{\frac{c}{\pi \alpha_1^2}}   \right)} \left[-\frac{1}{\sqrt{(\pi)}}\exp\left(-\frac{c}{\pi\alpha_1^2} \right)  \right]\left(\frac{1}{\alpha_1^2}\sqrt{\frac{c}{\pi}} \right)\\
&=& n\ln \left[1 - \frac{1}{2}\erf\left(-\sqrt{\frac{c}{\pi \alpha_1^2}}   \right)\right] - \frac{n\alpha_1}{1 - \frac{1}{2}\erf\left(-\sqrt{\frac{c}{\pi \alpha_1^2}}   \right)} \left[\frac{1}{\sqrt{(\pi)}}\exp\left(-\frac{c}{\pi\alpha_1^2} \right)  \right]\left(\frac{1}{\alpha_1^2}\sqrt{\frac{c}{\pi}} \right)
\end{eqnarray*}
Now since $ c = \alpha_1(1+\alpha_2) $ and we are in the regime $ \alpha_1 \le 1 + \alpha_2, \implies \frac{c}{\alpha_1^2} \ge 1 $. Hence
\begin{eqnarray*}
\frac{1}{S}\frac{\partial S}{\partial \alpha_1} &\ge& n\ln \left[1 - \frac{1}{2}\erf\left(-\sqrt{\frac{1}{\pi}}   \right)\right] - \frac{\frac{n}{\sqrt{\pi}}}{1 - \frac{1}{2}\erf\left(-\sqrt{\frac{1}{\pi}}   \right)} \underbrace{ \left[\left(\sqrt{\frac{c}{\pi\alpha_1^2}} \right)\exp\left(-\frac{c}{\pi\alpha_1^2} \right)  \right]}_{x\exp(-x^2) \le 0.428} \\
&=& 0.252n - 0.187n\\ 
\implies \frac{\partial S}{\partial \alpha_1} &>& 0
\end{eqnarray*}

Similarly we can show that in the \textbf{Regime $ \alpha_1 > 1 + \alpha_2 $}, $ \frac{\partial S}{\partial \alpha_2} > 0 $ which would imply $ \frac{\partial S}{\partial \alpha_1} < 0 $.
%\end{comment}

%We observe that by definition, the expression $ \sqrt{\frac{u}{\pi l - 4}} \ge \sqrt{\frac{1}{\pi}} $, and the equality is attained when $ u = l (\because u,l \gg 4)$ so the base $ B = 1 - \frac{1}{2}\erf \left( - \sqrt{\frac{u}{\pi l - 4}}\right) $ satisfies $ 1.29 \le B \le 1.5 $. Thus if $ \alpha_1 < 1 + \alpha_2 $, then $ 1.29^{\alpha_1n} \le S \le 1.5^{\alpha_1n} $, which means $ \C(n,\alpha_1, \alpha_2) $ is an increasing function of $ \alpha_1 $. If $ \alpha_1 > 1 + \alpha_2 $, then $ 1.29^{\frac{cn}{\alpha_1}} \le S \le 1.5^{\frac{cn}{\alpha_1}} $, which means $ \C(n,\alpha_1, \alpha_2) $ is a decreasing function of $ \alpha_1 $. 
Hence the maximum occurs when either $ \alpha_1 = 1 + \alpha_2 = \sqrt{c} $ ($ c \ge 1 $) or $ \alpha_1= c $ ($ c < 1 $).
\end{proof}

\section{Relationship between modes of joint and marginal distribution}
\textbf{Proposition}\\
Let $ \{\textbf{v}_r\}_{r=1}^k $ be visible vectors such that for each pair of vectors $ \{\textbf{v}_i,\textbf{v}_j\} $ in $ \{\textbf{v}_r\}_{r=1}^k $, $ d_H(\textbf{v}_i, \textbf{v}_j) \ge 2$. For an $\RBM_{n,m_1,...,m_L}(\theta) $ that fits the input distribution $ p(\textbf{v}) = \frac{1}{k}\sum_{i=1}^k \delta(\textbf{v} - \textbf{v}_i) $, if a vector $ \textbf{v} $ is a mode of marginal distribution, then there exist vectors $ \{\textbf{h}^*_l\}_{l=1}^L $ such that $(\textbf{v},\{\textbf{h}^*_l\}_{l=1}^L) $ is a mode of joint distribution $ p(\textbf{v},\{\textbf{h}_l\}_{l=1}^L) $.
\begin{proof}
Since $ v $ is a mode,  $\implies p(\textbf{v}) = \frac{1}{k} > 0$. \\
Further, let $ \{\textbf{h}^*_l\}_{l=1}^L = \argmax_{\{\textbf{h}_l\}} P(\textbf{v},\{\textbf{h}_l\}_{l=1}^L)$, that is, the state $ (\textbf{v},\{\textbf{h}^*_l\}_{l=1}^L) $ is stable against flip of any \textbf{hidden} unit\footnote{Here we assume that energy function values of any two distinct configurations are different.}.
\\
Moreover, since for all neighbours $ \textbf{v}' $ of $ \textbf{v} $, $ p(\textbf{v}') = 0 \implies p(\textbf{v}',\{\textbf{h}^*_l\}_{l=1}^L) = 0$, it implies that $(\textbf{v},\{\textbf{h}^*_l\}_{l=1}^L) $ is stable against flip of any visible unit also.

Thus $(\textbf{v},\{\textbf{h}^*_l\}_{l=1}^L) $ is one-flip stable and hence a mode of the joint distribution.
\end{proof}

\begin{figure*}[!t]
\centering
\subfloat[$ 2^{n\C(n,m_1,m_2)} $ for $ n = 3 $]{\includegraphics[width=2.0in]{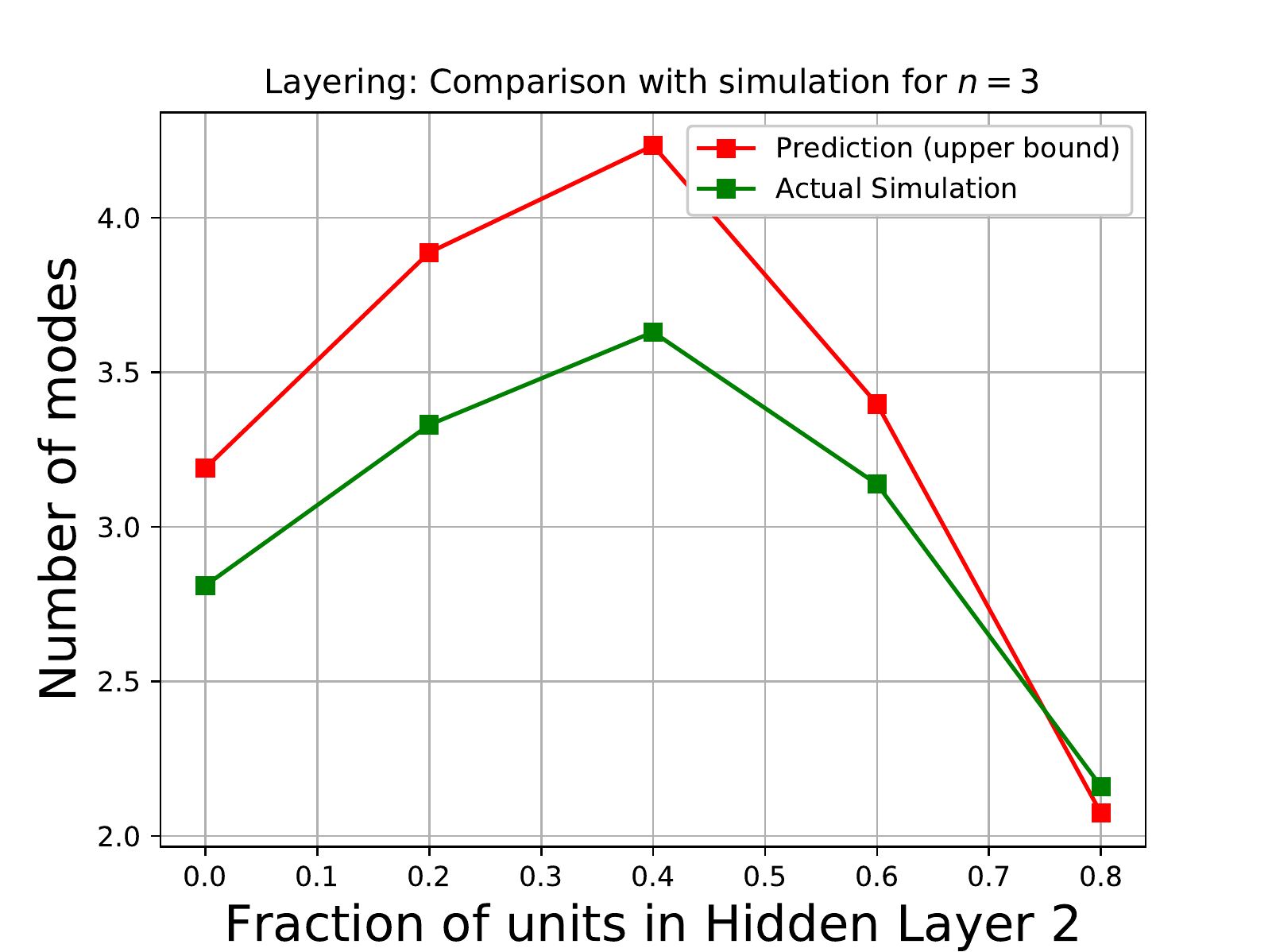}}%
\hfil
\subfloat[$ 2^{n\C(n,m_1,m_2)} $ for $ n = 5 $]{\includegraphics[width=2.0in]{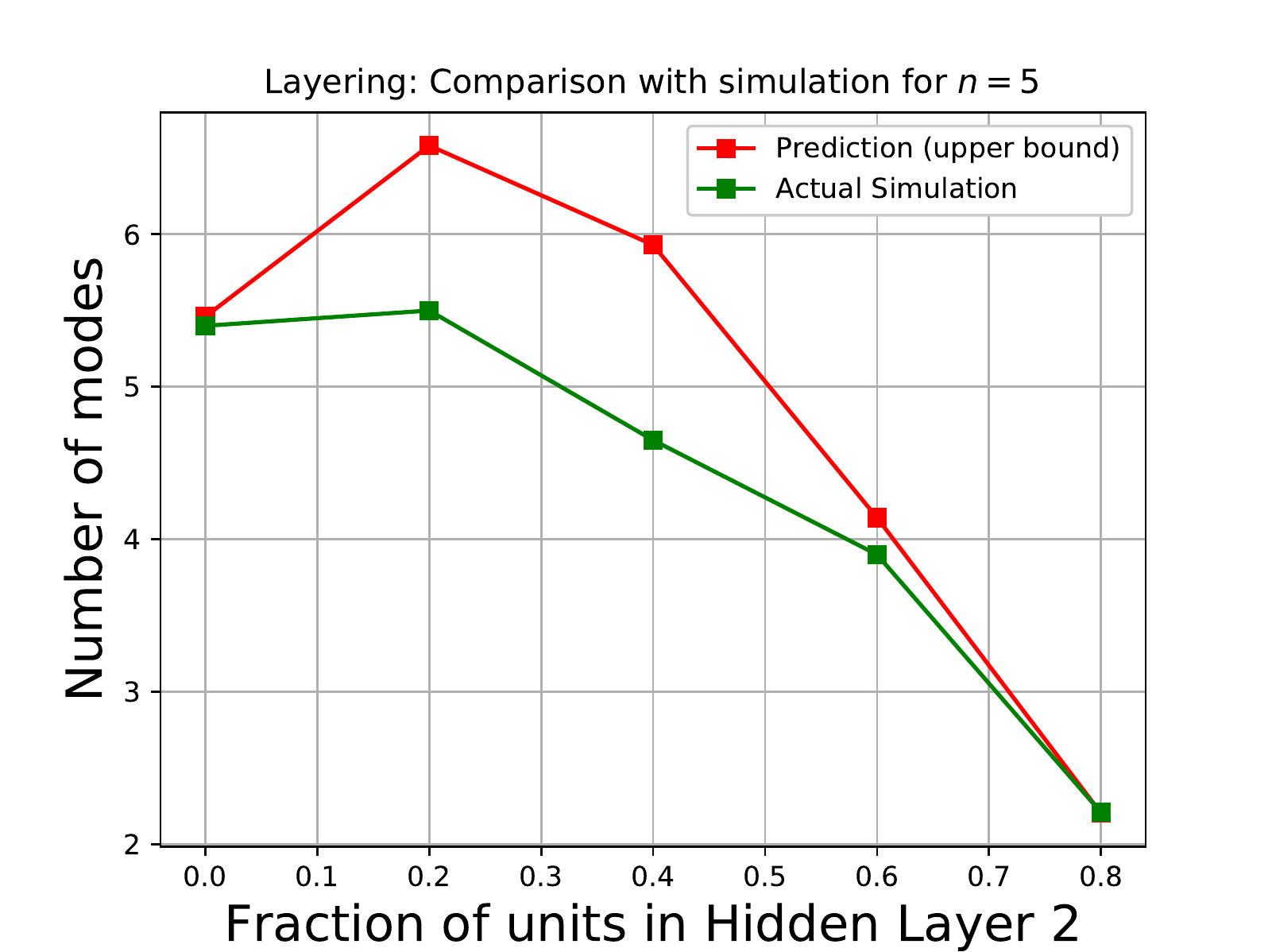}}%
\hfil
\subfloat[$ 2^{n\C(n,m_1,m_2)} $ for $ n = 10 $]{\includegraphics[width=2.0in]{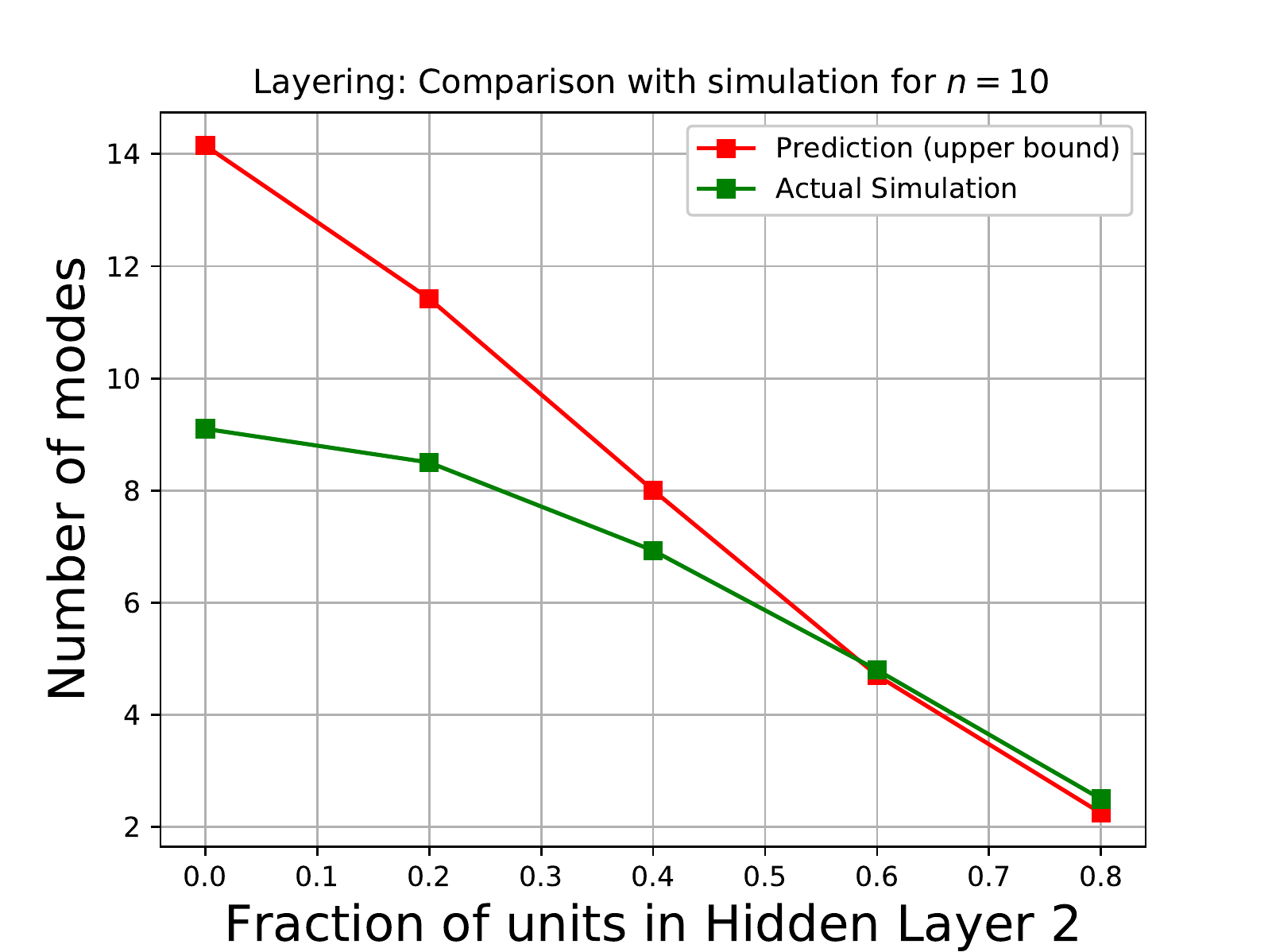}}%
\hfil
\caption{Comparison chart of the upper estimates with the actual simulation value for two layered RBM with $ m_1 + m_2 = 10 $. The values are plotted for various values of $ \beta = \frac{m_2}{m_1} $.}
\label{fig-comparison-layers}
\end{figure*}

\end{document}